%% file: structured_bandits.tex
\documentclass[11pt]{article}

\usepackage[in]{fullpage}
\usepackage{amsmath}
\usepackage{amssymb}
\usepackage{theorem}
\usepackage{algorithmic}
\usepackage{algorithm}
\usepackage{bbm}
\usepackage{bm}
\usepackage{enumerate}
\usepackage{color}
\usepackage{array}
\usepackage{multirow}
\usepackage{hyperref}
\usepackage[square,numbers]{natbib}

\input{notation}

\newcommand{\betaval}{\ensuremath C\psi(E_{r,\max})(w(\Omega_R)+\sqrt{\gamma^2+\log{T}}\phi(\Omega_R)/2)}
\newcommand{\lambdaval}[1]{\ensuremath c_0 (w(\Omega_R)+\sqrt{\gamma^2+\log{T}})/\sqrt{#1}}

\makeatletter
\newenvironment{theorepeat}[1]{\textbf{\@begintheorem{#1}}{\unskip}\itshape}{\@endtheorem}
\makeatother

\begin{document}

\title{Structured Stochastic Linear Bandits}

\author{Nicholas Johnson, Vidyashankar Sivakumar, Arindam Banerjee \vspace{5pt} \\
\{njohnson,sivakuma,banerjee@cs.umn.edu\} \vspace{5pt} \\
Department of Computer Science and Engineering \\
University of Minnesota, Twin Cities}

\maketitle

\begin{abstract}
\input{abstract}
\end{abstract}

\section{Introduction} \label{sec:intro}
\input{intro}

\section{Background: High-Dimensional Structured Estimation} \label{sec:est}
\input{est}

\section{Structured Bandits: Problem and Algorithm} \label{sec:setting}
\input{setting}

\subsection{Algorithm} \label{sec:alg}
\input{alg}

\section{Regret Bound for Structured Bandits} \label{sec:results}
\input{results}
\input{examples}

\section{Overview of the Analysis} \label{sec:analysis}
\input{analysis}


\section{Conclusions} \label{sec:conc}
\input{conc}

\newpage
\appendix
{\LARGE\textbf{Appendix}}
\input{appendix}

\bibliographystyle{plainnat}
\bibliography{library,structured_bandits}

\end{document}

%% file: notation.tex
\usepackage{mathtools}
\usepackage{theorem}

\newcommand{\bea}{\begin{equation}}
\newcommand{\eea}{\end{equation}}

\newcommand{\bes}{\begin{split}}
\newcommand{\ees}{\end{split}}

\newcommand{\benum}{\begin{enumerate}}
\newcommand{\eenum}{\end{enumerate}}

\newcommand{\bigO}{\ensuremath{O}}
\newcommand{\bigOtilde}{\ensuremath{\widetilde{O}}}

\newcommand{\cX}{{\cal X}}

\newcommand{\myref}[1]{(\ref{#1})}

\DeclareMathOperator{\argmax}{argmax}
\DeclareMathOperator{\argmin}{argmin}

        {\theorembodyfont{\rmfamily} \newtheorem{exm}{Example}}
        {\theorembodyfont{\rmfamily} }
        {\theorembodyfont{\rmfamily} \newtheorem{defn}{Definition}}
        {\theorembodyfont{\rmfamily} \newtheorem{asmp}{Assumption}}
        \newtheorem{theo}{Theorem}
        
        \newtheorem{lemm}{Lemma}

\newcommand{\qed}{\nolinebreak[1]~~~\hspace*{\fill} \rule{5pt}{5pt}\vspace*{\parskip}\vspace*{1ex}}

\newcommand {\commentout}[1] {}
\newcommand{\defeq}{\vcentcolon=}

\newcommand{\subgnorm}[1]{\ensuremath |\!\|#1\|\!|_{\psi_2}}

\newcommand{\trace}{\text{trace}}

%% file: abstract.tex
The stochastic linear bandit problem proceeds in rounds where at each round the
algorithm selects a vector from a decision set after which it receives
a noisy linear loss parameterized by an unknown vector. The goal in such a
problem is to minimize the (pseudo) regret which is the difference between
the total expected loss of the algorithm and the total expected loss of the
best fixed vector in hindsight. In this paper, we consider settings where
the unknown parameter has structure, e.g., sparse,
group sparse, low-rank, which can be captured by a norm, e.g., $L_1$, $L_{(1,2)}$,
nuclear norm. We focus on constructing confidence ellipsoids which contain the unknown
parameter across all rounds with high-probability. We show the radius of such
ellipsoids depend on the Gaussian width of sets associated with the
norm capturing the structure.
Such characterization leads to tighter confidence ellipsoids
and, therefore, sharper regret bounds compared to bounds in the existing literature
which are based on the ambient dimensionality.


%% file: intro.tex
We consider the stochastic linear bandit problem~\cite{dahk08,abps11} which proceeds
in rounds $t=1, \dots, T$ where at each round $t$ the algorithm selects a vector
$x_t$ from some decision set $\mathcal{X} \subset \mathbb{R}^p$
and receives a noisy loss defined as
$\ell_t(x_t) = \langle x_t, \theta^* \rangle + \eta_t$ where $\theta^*$ is
an unknown parameter and $\eta_t$ is martingale noise. The algorithm observes only
$\ell_t(x_t)$ at each round $t$ and its goal is to minimize the cumulative loss.
We measure its performance by the (pseudo) regret~\cite{buce12} defined as
\begin{equation} \label{eq:pseudo_regret:sec:intro}
R_T = \sum_{t=1}^T \langle x_t, \theta^* \rangle - \underset{x^* \in \mathcal{X}}{\argmin} \sum_{t=1}^T \langle x^*, \theta^* \rangle~.
\end{equation}

The stochastic linear bandit can be used to model problems in several real-world
applications ranging from recommender systems to medical treatments to
network security. Frequently, in such applications, one has knowledge of the
structure of the unknown parameter $\theta^*$, for example, $\theta^*$ may be
sparse, group sparse, or low-rank. Previous works~\cite{dahk08,abps11} either
made no structural assumptions on $\theta^*$
and proved regret bounds\footnote{The $\bigOtilde(\cdot)$ notation selectively
hides constants and log terms.} of the form
$\bigOtilde(p\sqrt{T})$ or assumed $\theta^*$ was $s$-sparse ($s$ non-zero
elements) and showed~\cite{abps12} the regret sharpens to
$\bigOtilde(\sqrt{spT})$.
In this paper, we consider the setting where $\theta^*$ is any
generally structure vector (sparse, group sparse, low-rank, etc.)
such that the structure can be captured by some norm
 ($L_1$, $L_{(1,2)}$, nuclear norm, etc.).

Our approach follows previous works~\cite{dahk08,abps11,abps12}
which use the \emph{optimism-in-the-face-of-uncertainty} principle~\cite{buce12}
to design a class of algorithms which construct a confidence ellipsoid $C_t$
such that $\theta^* \in C_t$ across all rounds with high-probability. After
which, the algorithm selects a single $x_{t+1}$ by solving a bilinear
optimization problem with respect to parameters $x \in \mathcal{X}$ and $\theta \in C_t$.

Our algorithm differs from previous algorithms~\cite{dahk08,abps11,abps12} in
two key ways. First, for the initial rounds, we select random samples from the
decision set $\mathcal{X}$ to compute an estimate of $\theta^*$ such that the
estimate is statistically consistent. The length of
the random estimation rounds is dependent on the structure, where, for example,
if $\theta^*$ is $s$-sparse
scales like $s\log{p}$ and if $\theta^*$ is unstructured
scales like $p$.
Second, after the random estimation rounds, we select samples uniformly at random
from specific subsets of $\mathcal{X}$.
More specifically, previous works selected a sample by solving a bilinear optimization problem
however, such an approach only gives one, possibly unique, solution. We build on
such works by solving a similar bilinear optimization problem but then
center two $L_2$ balls of suitable radii over the parameters $x$ and $\theta$
and select $x_{t+1}$
uniformly at random using such balls.

\textbf{Overview of Results.} The main technical challenge in previous
works~\cite{dahk08,abps11,abps12} is constructing confidence ellipsoids which
contain $\theta^*$ across all rounds with high-probability.
The focus of our work is again to construct confidence ellipsoids such that
$\theta^* \in C_t$ across all rounds with high-probability but which are general
enough to hold for any norm structured $\theta^*$. Moreover, we desire that the
ellipsoids are tighter than previous works in order to provide sharper regret bounds.
Previous works~\cite{dahk08,abps11} constructed the confidence ellipsoids by
solving a ridge regression problem to compute an estimate $\hat{\theta}_t$
and centered an ellipsoid over the estimate. We generalize such an approach
by instead solving a norm regularized regression problem, e.g., Lasso, given
the structure of $\theta^*$.
We show our construction of $C_t$ contains $\theta^*$ across all rounds
with high-probability by extending recent results in structured
estimation~\cite{cata07,crpw12,newa11,bcfs14} which rely on i.i.d. samples
to active sampling.

The main technical result we show is that the radius of our confidence
ellipsoids depend on the Gaussian width\footnote{The Gaussian width is a
geometric characterization of the size of a set and the definition is presented
in Section~\ref{sec:appendix_defns}.}
of sets associated with the structure of $\theta^*$ which leads to tighter
confidence ellipsoids than previous works~\cite{dahk08,abps11} when $\theta^*$ is
structured. For example, with an $s$-sparse $\theta^*$ the radius of the
confidence ellipsoid scales as $\bigO(\sqrt{s\log{p}})$ compared to
$\bigO(\sqrt{p})$ in the unstructured settings considered in~\cite{dahk08,abps11}.

The regret bounds for our algorithm follow from the analysis in~\cite{dahk08} and depend on the
radius of the confidence ellipsoid therefore, our regret bounds scale with
the structure of $\theta^*$ as measured by the Gaussian width which leads to
sharper regret bounds when $\theta^*$ is generally structured and matches
existing bounds when $\theta^*$ is $s$-sparse~\cite{abps12} or
unstructured~\cite{dahk08}.

\subsection{Previous Works}
Multiarmed bandits have a mature and active literature~\cite{buce12}.
One popular algorithm based on the OFU principle and upper confidence bounds
is UCB~\citep{aucf02} used for the $K$-arm stochastic bandit problem where the
algorithm selects a decision $k \in \{1,\dots,K\}$ and receives a stochastic
loss drawn i.i.d. from the $k$th decision's distribution.
The regret was shown to be $\bigO(\frac{K}{\Delta}\log{T})$ where $\Delta$ is
the gap in performance between the best and second best decisions.

A similar problem has been considered~\cite{auer02,lcls10,clrs11} when a
$p$-dimensional feature vector is provided for each of the $K$ decisions and
the expected loss is a linear function of the feature vector and an unknown
parameter. For such a problem, a regret bound of
$\bigO(\log^{3/2}(K) \sqrt{p} \sqrt{T})$ was shown. However,
dependence on the number of decisions $K$ can be problematic when it is large or infinite.

For such settings, \cite{dahk08} studied the stochastic linear bandit problem where
the decision set is an arbitrary compact set in $\mathbb{R}^p$. They
presented the algorithm ConfidenceBall2 based on the OFU principle which
computes an estimate $\hat{\theta}_t$ of the unknown parameter $\theta^*$ using
ridge regression and constructs an ellipsoidal confidence set around
$\hat{\theta}_t$ with radius $\sqrt{\beta_t}$. After which
the vectors $x_t$ and $\widetilde{\theta}_t$ are selected optimistically from
the decision set and confidence ellipsoid respectively, such that $\langle x_t, \widetilde{\theta}_t \rangle$
is minimized. They showed how to set $\beta_t$ such that $\theta^*$
stays within the ellipsoid with high-probability for all $t$ and showed a problem independent
regret\footnote{The $\widetilde{\bigO}(\cdot)$ notation selectively hides
constant and log factors.} of $\widetilde{\bigO}(p\sqrt{T})$ and a problem
dependent regret of $\bigOtilde(\frac{p^2}{\Delta} \log^3{T})$ where $\Delta>0$
is the gap between the best and second best extremal points.
Note, the regret depends on the ambient dimensionality $p$ and not the
number of decisions. Further,~\cite{ruts10,abps11} showed how to
construct tighter confidence ellipsoids, specifically,~\cite{abps11} used a
self-normalized tail inequality for vector-valued martingales which decreased
the regret by a $\sqrt{\log{T}}$ multiplicative factor.

Building on such works which had considered the problem without structural
assumptions on $\theta^*$, two papers published
simultaneously~\citep{camu12,abps12} considered the problem where $\theta^*$ is
$s$-sparse.~\cite{abps12} followed the same
problem setting as~\cite{dahk08} and presented a method which can use the
predictions of any full information online algorithm with an upper bound on its
regret to construct confidence sets. When constructing confidence sets using the
algorithm SeqSEW~\citep{gerc11} they showed a problem independent regret of
$\widetilde{\bigO}(\sqrt{spT})$ and a problem dependent regret of
$\widetilde{\bigO}(\frac{sp}{\Delta}\log^2{T})$.

\cite{camu12} uses initial rounds to estimate the sparse structure similar to our
work however, they do not consider the standard stochastic linear bandit problem.
Specifically, they define the loss as
$\ell_t(x_t) = \langle x_t, \theta^* \rangle + \langle x_t, \eta_t \rangle$
where $\eta_t$ is i.i.d. whitenoise (not martingale) and they assume the decision
set is the unit $L_2$ ball. During the random estimation rounds, they used
techniques from compressed sensing to identify the subspace where $\theta^*$
lives then ran ConfidenceBall2 where the decision set is a subset of the subspace
and showed a problem independent regret of $\widetilde{\bigO}(s\sqrt{T})$.

Such papers show that sharper regret bounds can be obtained when $\theta^*$ is
structured however, only for a sparse $\theta^*$. Such results motivate
our work to study the regret for any generally norm structured $\theta^*$.
We organize the paper as follows. In Section~\ref{sec:est} we give background
on high-dimensional structured estimation which our analysis builds on. Section~\ref{sec:setting}
we present the problem setting and algorithm. Section~\ref{sec:results} we present
our main regret bounds from a high-level and provide examples of popular types of structure.
Section~\ref{sec:analysis} we present the main technical results of the analysis
and point to the detailed proofs in the appendix. Finally, we conclude in
Section~\ref{sec:conc}.

%% file: est.tex
We rely on recent developments in the analysis of non-asymptotic
bounds for structured estimation in high-dimensional statistics. In this
section, we will discuss the main results needed for our analysis which can be
found in the following papers~\citep{cata07,birt09,crpw12,nrwy12,vers12,bolm13,bcfs14,bcfs15}.

In high-dimensional structured estimation, one is concerned with settings in
which the dimension $p$ of the parameter $\theta^*$ to be estimated is
significantly larger than the sample size $n$, i.e., $p \gg n$. It is known
that for $n$ i.i.d. Gaussian samples, one can compute an estimate $\hat{\theta}_n$
using least squares regression which converges to $\theta^*$ at a rate of
$\bigO\left(\sqrt{\frac{p}{n}}\right)$. The convergence rate can be improved
when $\theta^*$ is structured which is usually characterized as having a small value
according to some norm $R(\cdot)$. For such problems, estimation
is performed by solving a norm regularized regression problem
\begin{equation} \label{eq:norm_reg:sec:est}
\hat{\theta}_n \defeq \underset{\theta \in \mathbb{R}^p}{\argmin} \; \mathcal{L}(\theta, Z_n) + \lambda_n R(\theta)
\end{equation}
where $\mathcal{L}(\cdot, \cdot)$ is a convex loss function\footnote{We drop the
second argument when it is clear from the context.}, $Z_n$ is a
dataset consisting of i.i.d. pairs $\left\{(x_i, y_i)\right\}_{i=1}^n$ where
$x_i \in \mathbb{R}^p$ is a sample, $y_i \in \mathbb{R}$ is the response, and
$\lambda_n$ is the regularization parameter.

For such problems, let $\hat{\theta}_n-\theta^*$ be the
estimation error vector, then for a suitably large $\lambda_n$, \cite{bcfs14}
showed the error vector deterministically belongs to the restricted error set
\begin{equation} \label{eq:E_r:sec:est}
E_{r,n} = \left\{\hat{\theta}_n-\theta^* \in \mathbb{R}^p: R(\hat{\theta}_n) \leq R(\theta^*) + \frac{1}{\rho} R(\hat{\theta}_n-\theta^*) \right\}
\end{equation}
where $\rho > 1$ is a constant which we fix as $\rho=2$ for ease of exposition.
For such a $\rho$, $E_{r,n}$ is a restricted set
of directions, in particular, the error vector $\hat{\theta}_n-\theta^*$ cannot
be in the direction of $\theta^*$.
Using the restricted error set, bounds on the estimation error can be
established which hold with high-probability under two assumptions. First, the
regularization parameter $\lambda_n$ must satisfy the inequality
\begin{equation} \label{eq:lambda_n:sec:est}
\lambda_n \geq 2 R^*(\nabla \mathcal{L}(\theta^*, Z_n))
\end{equation}
where $R^*(\cdot)$ is the dual norm of $R(\cdot)$.
Second, the loss function must satisfy the restricted strong convexity (RSC) condition
in the restricted error set $E_{r,n}$ as illustrated in~\citep{nrwy12}. Specifically,
there exists a $\kappa>0$ such that
\begin{equation}
\mathcal{L}(\hat{\theta}_n) - \mathcal{L}(\theta^*) - \langle \nabla \mathcal{L}(\theta^*), \hat{\theta}_n-\theta^* \rangle \geq \kappa \|\hat{\theta}_n-\theta^*\|_2^2 \quad \forall \hat{\theta}_n-\theta^* \in E_{r,n}~.
\end{equation}
For the squared loss, the RSC condition simplifies to the
restricted eigenvalue (RE) condition
\begin{equation} \label{eq:re_eq:sec:est}
\frac{1}{n} \|X_n (\hat{\theta}_n-\theta^*)\|_2^2 \geq \kappa \|\hat{\theta}_n-\theta^*\|_2^2 \quad \forall \hat{\theta}_n-\theta^* \in E_{r,n}
\end{equation}
where $X_n \in \mathbb{R}^{n \times p}$ is the design matrix~\citep{crpw12}.
Under such conditions,
the following bound holds with high-probability~\citep{nrwy12,bcfs14}
\begin{equation} \label{eq:delta_n:sec:est}
\|\hat{\theta}_n-\theta^*\|_2 \leq c \psi(E_{r,n}) \frac{\lambda_n}{\kappa}
\end{equation}
where $\psi(E_{r,n}) = \sup_{u \in E_{r,n}} \frac{R(u)}{\|u\|_2}$ is the norm compatibility
constant and $c > 0$ is a constant.
For an $s$-sparse $\theta^*$, one obtains
$\|\hat{\theta}_n-\theta^*\|_2 \leq \bigO\left(\sqrt{\frac{s\log{p}}{n}}\right)$ and
for a group sparse $\theta^*$, one obtains 
$\|\hat{\theta}_n-\theta^*\|_2 \leq \bigO\left(\sqrt{\frac{s_\mathcal{G}(m+\log{K})}{n}}\right)$
where $K$ is the number of groups, $m$ is the maximum group size, and $s_\mathcal{G}$
is the group sparsity level. Similar bounds can be computed for other types of
structure including low-rank.

%% file: setting.tex
Here, we will formally define the problem, mention the
assumptions under which our analysis works, and present our algorithm. The results
and analysis are presented in subsequent sections.

\subsection{Problem Setting} \label{ssec:problem}
We consider the stochastic linear bandit problem~\cite{dahk08,abps11} where in
each round $t=1,\dots,T$ the algorithm selects a $p$-dimensional vector $x_t$ from the
decision set $\mathcal{X}$ 
and receives a loss of $\ell_t(x_t) = \langle x_t, \theta^* \rangle + \eta_t$.
Our focus is on settings where the unknown parameter $\theta^* \in \mathbb{R}^p$
is structured which we characterize as having a small value according to some norm
$R(\cdot)$.

The goal of the algorithm is to minimize its cumulative
loss $\sum_t \ell_t(x_t)$ and we measure the performance of the algorithm in terms
of the fixed cumulative (pseudo) regret defined as
\begin{equation} \label{eq:regret}
R_T = \sum_{t=1}^T \langle x_t, \theta^* \rangle - \underset{x^* \in \mathcal{X}}{\min} \sum_{t=1}^T \langle x^*, \theta^* \rangle~.
\end{equation}
We require that the algorithm's regret grows sub-linearly in $T$, i.e.,
$R_T = o(T)$, and desire it grows with the structure of $\theta^*$ rather
than the ambient dimensionality $p$ with high-probability. The following
assumptions under which our analysis holds are standard in the literature~\cite{dahk08,abps11,abps12}.

\subsection{Assumptions and Definitions} \label{ssec:assumptions}
\input{asmps_defns}


%% file: asmps_defns.tex
\begin{asmp} \label{asmp:decision_set:sec:asmps_defns}
The decision set $\mathcal{X} \subset \mathbb{R}^p$ is a compact (closed and
bounded) convex set with non-empty interior.
For ease of exposition, we assume $\mathcal{X} \subseteq \bar{B}_2^p$, the
(closed) unit $L_2$ ball defined as $\bar{B}_2^p = \{x \in \mathbb{R}^p: \|x\|_2 \leq 1\}$,
to avoid scaling factors.

\end{asmp}
\vspace{-15pt}
\begin{asmp} \label{asmp:noise:sec:asmps_defns}
The noise is a bounded martingale difference sequence (MDS), i.e.,
$|\eta_t| \leq B, \mathbb{E}[\eta_t] < \infty, \mathbb{E}[\eta_t | F_{t-1}] = 0 \; \forall t$
where $F_t = \{x_1, \dots, x_{t+1}, \eta_1, \dots, \eta_t\}$ is a filtration
(sequence of $\sigma$-algebras).
We assume bounded noise for simplicity however, the results hold for any
sub-Gaussian noise (refer to Section~\ref{sec:appendix_defns} for definitions
of sub-Gaussian and related quantities).
\end{asmp}
\vspace{-15pt}
\begin{asmp} \label{asmp:structure:sec:asmps_defns}
We assume the unknown parameter $\theta^*$ is fixed for all rounds, the
structure is known, for example, for an $s$-sparse
$\theta^*$ the value of $s$ is known,
and $\|\theta^*\|_2 = 1$.
\end{asmp}
\vspace{-15pt}
\begin{asmp} \label{asmp:T:sec:asmps_defns}
The number of rounds $T$ is known a priori.
\end{asmp}

\begin{defn} \label{defn:gaussian_width:sec:asmps_defns}
The Gaussian width~\cite{crpw12} of a set $A$ is defined as
$w(A) = \mathbb{E}\left[\sup_{u \in A} \langle g, u \rangle \right]$
where the expectation is over $g$ which is a zero mean, unit variance
Gaussian random variable.
\end{defn}
\vspace{-15pt}
\begin{defn} \label{defn:diameter:sec:asmps_defns}
For a set $A$, $\phi(A) = \sup_{u,v \in A} \|u-v\|_2 = \sup_{u \in A} \|u\|_2$
measures the diameter of $A$.
\end{defn}
\vspace{-15pt}
\begin{defn} \label{defn:E_r:sec:asmps_defns}
The restricted error set is defined as
$E_{r,t} \defeq \left\{\hat{\theta}_t - \theta^* \in \mathbb{R}^p: R(\hat{\theta}_t) \leq R(\theta^*) + \frac{1}{2}R(\hat{\theta}_t-\theta^*)\right\}$
and the largest such error set is
$E_{r,\max} = \argmax_{E_r \in \{E_{r,1}, \dots, E_{r,T}\}} w(E_r)$.
\end{defn}
\vspace{-15pt}
\begin{defn} \label{defn:A:sec:asmps_defns}
The set $A_t$ is a spherical cap constructed as
$A_t \defeq \text{cone}(E_{r,t}) \cap S^{p-1}$ where $S^{p-1}$ is the unit sphere
in $p$-dimensions and $A_{\max} \defeq \text{cone}(E_{r,\max}) \cap S^{p-1}$
is the largest such cap.
\end{defn}
\vspace{-15pt}
\begin{defn} \label{defn:subg_norm_constant:sec:asmps_defns}
Each $x_t$ has sub-Gaussian norm (refer to Section~\ref*{sec:appendix_defns} for
the definition of sub-Gaussian norm)
satisfying $\subgnorm{x_t} \leq K$ for some
absolute constant $K$. This follows from Assumption~\ref{asmp:decision_set:sec:asmps_defns}.
\end{defn}
\vspace{-15pt}
\begin{defn} \label{defn:Omega_R:sec:asmps_defns}
The unit norm $R(\cdot)$ ball is $\Omega_R \defeq \{u \in \mathbb{R}^p: R(u) \leq 1 \}$.
The norm compatibility constant with respect to vectors in the restricted error set
at round $t$ is $\psi(E_{r,t}) = \sup_{u \in E_{r,t}} \frac{R(u)}{\|u\|_2}$.
\end{defn}

%% file: alg.tex
For the initial $t=1,\dots,n=c'w^2(A_{\max})(\epsilon^2+\log{T})$ rounds where
$c'>0$ is a constant, our algorithm selects vectors
$x_{1:n} \defeq \{x_1, \dots, x_n\}$ uniformly at random from $\mathcal{X}$ and
receives the corresponding losses
$\ell_{1:n} \defeq \{\ell_1(x_1), \dots, \ell_n(x_n)\}$. The length of such
random estimation rounds depends on the Gaussian width of the largest
spherical cap induced by the structure of $\theta^*$ and a parameter $\epsilon>0$ which
controls the success probability and will become clear in the analysis in
Section~\ref{sec:analysis}. The random estimation rounds can be considered the
``burn-in'' period similar to the use of a barycentric spanner or identity
matrix as in~\citep{dahk08,abps11}.

After the loss $\ell_n(x_n)$ is received in round $n$, the algorithm
constructs an $(n\times p)$-dimensional design matrix $X_n=[x_1 \dots x_n]^\top$,
a sample covariance matrix $D_n = X_n^\top X_n$,
and an $n$-dimensional response vector $y_n=[\ell_1(x_1) \dots \ell_n(x_n)]^\top$.
The algorithm then computes an estimate $\hat{\theta}_n$ by solving a norm
regularized regression problem, constructs a confidence ellipsoid using the
Mahalanobis distance defined as
$\|\theta - \hat{\theta}_n\|_{2,D_n} = \sqrt{(\theta-\hat{\theta}_n)^\top D_n (\theta-\hat{\theta}_n)}$,
then selects a sample to play. Specifically, the algorithm performs the
following four main steps sequentially in each round thereafter.

For each $t=n,\dots,T$:
\begin{align}
&\text{1. Compute an estimate:} \hspace{72pt}\hat{\theta}_t \defeq \underset{\theta \in \mathbb{R}^p}{\argmin} \; \frac{1}{t} \|y_t-X_t\theta\|_2^2 + \lambda_t R(\theta) \label{alg:compute_theta_hat:sec:alg} &~~~~~& \\[0.25ex]
&\text{2. Construct a confidence ellipsoid:} \hspace{20pt} C_t \defeq \left\{ \theta \in \mathbb{R}^p : \|\theta - \hat{\theta}_t\|_{2,D_t} \leq \beta \right\} \label{alg:construct_C_t:sec:alg} &~~~~~& \\[1ex]
&\text{3. Compute an optimal solution:} \hspace{43pt} (x_{t+1}', \theta_{t+1}') \defeq \underset{\substack{x \in \mathcal{X} \\ \theta \in C_t \cap S^{p-1}}}{\argmin} \; \langle x, \theta \rangle \label{eq:arm_select_round_t:sec:alg} &~~~~~& \\
&\text{4. Play vector} \; x_{t+1} \sim \text{Uniform}(\mathcal{X} \cap \bar{B}_2^p(x_{t+1}',\|x_{t+1}'\|_2/2)) \; \text{and receive loss} \; \ell_{t+1}(x_{t+1}) \notag &~~~~~&
\end{align}
where $\bar{B}_2^p(x_{t+1}',\|x_{t+1}'\|_2/2))$ is a closed $L_2$ ball centered
at $x_{t+1}'$ with radius $\|x_{t+1}'\|_2/2$. After receiving the loss
$\ell_{t+1}(x_{t+1})$, the design matrix $X_{t+1}$ and response vector $y_{t+1}$
are updated with $x_{t+1}$ and $\ell_{t+1}(x_{t+1})$ respectively. Then, the sample
covariance matrix $D_{t+1} = X_{t+1}^\top X_{t+1}$ is recomputed and the regularization
parameter $\lambda_{t+1}$ is updated.

\begin{algorithm}[t!]
\caption{\small{Structured Stochastic Linear Bandit}}
\label{alg:gen_struct_bandit:sec:alg}
\begin{algorithmic}[1]
\small
\STATE Input: $p$, $\mathcal{X}$, $R(\cdot)$, $T$, $E_{r,\max}$,$A_{\max}$, $\Omega_R$, $\gamma$, $\epsilon$, $c_0, c', C$
\STATE Set $\beta = \betaval$ \hfill (\ref{eq:beta_val:sec:analysis})
\STATE Play $n=c'w^2(A_{\max})(\epsilon^2+\log{T})$ uniform i.i.d. random vectors $x_{1:n} \in \mathcal{X}$ and receive losses $\ell_{1:n}$
\STATE For $t=n,\dots,T$
\STATE ~~ Compute $X_t=[x_1 \dots x_t]^\top$, $y_t=[\ell_1(x_1) \dots \ell_t(x_t)]^\top$, and $D_t=X_t^\top X_t$
\STATE ~~ Set $\lambda_t = \lambdaval{t}$ \hfill (\ref{eq:lambda_val:sec:analysis})
\STATE ~~ Compute $\hat{\theta}_t = \underset{\theta \in \mathbb{R}^p}{\argmin} \; \frac{1}{t}\|y_t-X_t\theta\|_2^2 + \lambda_t R(\theta)$
\STATE ~~ Construct $C_t \defeq \{\theta: \|\theta-\hat{\theta}_t\|_{2,D_t} \leq \beta \}$
\STATE ~~ Compute $(x_{t+1}', \theta_{t+1}') \defeq \argmin_{x \in \mathcal{X}, \, \theta \in C_t \cap S^{p-1}} \; \langle x, \theta \rangle$
\STATE ~~ Play $x_{t+1} \sim \text{Uniform}\left(\mathcal{X} \cap \bar{B}_2^p(x_{t+1}',\|x_{t+1}'\|_2/2)\right)$ and receive loss $\ell_{t+1}(x_{t+1})$
\STATE End For
\end{algorithmic}
\end{algorithm}

\subsubsection{Discussion}
\textbf{Step 1.} An estimate is computed by solving a norm regularized regression problem
following existing results discussed in Section~\ref{sec:est}. This generalizes
previous works~\cite{dahk08,abps11} which only consider computing an estimate
by solving the ridge regression problem. 

\textbf{Step 2.} A confidence ellipsoid is constructed in order to allow the algorithm to
explore in certain directions. Since the confidence ellipsoid is defined as
$C_t = \{\theta \in \mathbb{R}^p: \|\theta - \hat{\theta}_t\|_{2,D_t} \leq \beta\}$,
we focus on bounds for $\|\hat{\theta}_t - \theta^*\|_{2,D_t}$. Extending
the results in Section~\ref{sec:est}, we will show in Section~\ref{sec:analysis}
high-probability bounds on the estimation error of the form 
$\|\hat{\theta}_t - \theta^*\|_{2,D_t} \leq c \psi(E_{r,t})\frac{\lambda_t}{\kappa}\sqrt{t}$.
Therefore, setting $\beta$ to the right hand side will give bounds such that
$\theta^* \in C_t$ with high-probability. The value of $\beta$ then depends on two
key terms: the regularization parameter $\lambda_t$ and the restricted
eigenvalue (RE) constant $\kappa$ detailed in (\ref{eq:re_eq:sec:est}).
The value of $\lambda_t$ is set by the user
and we will provide an explicit characterization of its value in
Section~\ref{sec:analysis}. Moreover, the estimation error bound will
only hold when the RE constant $\kappa$ is positive and we will show in
Section~\ref{sec:analysis} that after a suitable number of random estimation
rounds and by selecting samples via Step 3, it will be positive for all rounds.

\textbf{Steps 3 and 4.} These steps are motivated from the regret analysis established in~\cite{dahk08} and
the need to satisfy the RE condition. Let the instantaneous regret at round $t+1$ be defined as
$r_{t+1} = \langle x_{t+1}, \theta^* \rangle - \langle x^*, \theta^* \rangle$
where $x^* = \argmin_{x \in \mathcal{X}} \langle x, \theta^* \rangle$. As shown
in~\cite{dahk08}, by selecting an $x_{t+1}$ and $\widetilde{\theta}_{t+1}$ via
\begin{equation} \label{eq:bilinear_opt:sec:alg}
(x_{t+1}, \widetilde{\theta}_{t+1}) \defeq \underset{\substack{x \in \mathcal{X} \\ \theta \in C_t}}{\argmin} \; \langle x, \theta \rangle
\end{equation}
the instantaneous regret can be upper bounded as
$r_{t+1} = \langle x_{t+1}, \theta^* \rangle - \langle x^*, \theta^* \rangle 
\leq \langle x_{t+1}, \theta^* \rangle - \langle x_{t+1}, \widetilde{\theta}_{t+1} \rangle$
because we optimize over both $x$ and $\theta$. Therefore, one obtains the following
inequality
$\langle x_{t+1}, \widetilde{\theta}_{t+1} \rangle \leq \langle x^*, \theta^* \rangle$
on which the entire regret analysis relies. We will use the regret analysis
from~\cite{dahk08} therefore, we need to select an $x_{t+1}$ and $\widetilde{\theta}_{t+1}$
such that the above inequality holds.

Additionally, recall the RE condition in (\ref{eq:re_eq:sec:est})
\begin{equation*}
\frac{1}{t} \|X_t(\hat{\theta}_t - \theta^*)\|_2^2 \geq \kappa \|\hat{\theta}_t-\theta^*\|_2^2 \quad \forall \hat{\theta}_t-\theta^* \in E_{r,t}~.
\end{equation*}
We must have $\kappa>0$ for the estimation error bound used to compute $\beta$
to hold. Therefore, in order to show such a $\kappa$ exists, we need samples
which are not too correlated otherwise the design matrix will be
ill-conditioned.

To use the regret analysis and satisfy the RE condition, we cannot exactly
follow existing work~\cite{dahk08,abps11,abps12} and select an $x_{t+1}$ by
solving (\ref{eq:bilinear_opt:sec:alg}) since we may obtain a
single unique solution and the rows of the design matrix will be too correlated.
Instead, we select samples uniformly at random from specific subsets of $\mathcal{X}$
which spreads the samples out enough to show the RE condition holds. Moreover,
as we will show in Section~\ref{sec:appendix_alg}, for any random sample
$x_{t+1}$ we select, we can deterministically compute a
$\widetilde{\theta}_{t+1} \in C_t$ such that the inequality
$\langle x_{t+1}, \widetilde{\theta}_{t+1} \rangle \leq \langle x^*, \theta^* \rangle$
holds.

Steps 1, 2, and 4 can be performed efficiently, in particular, there are several
efficient methods for computing the estimate in
(\ref{alg:compute_theta_hat:sec:alg}) for common regularizers, e.g., $L_1$,
$L_{(1,2)}$, nuclear norm, etc.~\citep{dadd04,pabo14,bpcp10}. Step 3 is
computationally difficult in general (similar to all previous work) however, for 
simple decision sets such as the unit the $L_2$ ball, a solution
can be computed efficiently by solving
the corresponding quadratically constrained quadratic program.
Our algorithm for structured stochastic linear bandits is presented in
Algorithm~\ref{alg:gen_struct_bandit:sec:alg}.

%% file: results.tex
Here, we present the main result which is a high-probability bound
on the regret of Algorithm~\ref{alg:gen_struct_bandit:sec:alg}
and show examples for popular types of structure.
The analysis of the bound is presented in Section~\ref{sec:analysis}. 

First, we review some of the assumptions from Section~\ref{ssec:assumptions}.
We assume $\mathcal{X}$ is a compact convex set with non-empty interior.
Further, we assume $\cX \subseteq B^p_2$, the unit $L_2$ ball, for ease of
exposition. Examples of such decision sets include:
$L_p$ balls for $1 \leq p < \infty$,
ellipsoids, polytopes, norm cones, and hypercubes.
We assume the noise $\eta_t$ is a bounded MDS where
$|\eta_t| \leq B, \forall t$ and the number of rounds $T$ is known a priori.

Further, we recall a few definitions introduced in
Section~\ref{ssec:assumptions} which will help interpret the main result. We
define the set $\Omega_R$ as the unit norm $R(\cdot)$ ball and
$A_{\max} \subset S^{p-1}$ as the spherical cap of the largest restricted error
set. For such sets, $w(\Omega_R)$ and $w(A_{\max})$ are the Gaussian widths.
Moreover, we define $\phi(\Omega_R)$ to be the diameter of the set $\Omega_R$,
$E_{r,\max}$ as the largest restricted error set, and $\psi(E_{r,\max})$ as the
norm compatibility constant of the largest restricted error set.

Under such assumptions, we present the main result in a high-level form, which
hides the exact nature of the constants involved. A more explicit form of the
constants is presented in the appendix.

The main result consists of two theorems for
the problem independent and problem dependent settings~\cite{dahk08}. Let $\mathcal{E}$
be the set of all extremal points. The problem independent
setting occurs when the difference between the expected loss of the best extremal
point $x^*$ and the expected loss of the second best extremal point is zero, i.e.,
$\Delta = \inf_{x \in \mathcal{E}} \langle x, \theta^* \rangle - \langle x^*, \theta^* \rangle = 0$.
Such a setting occurs, for example, when the decision set is the unit $L_2$ ball.
The problem dependent setting occurs when $\Delta > 0$, for example, when the
decision set is a polytope.

\begin{theo}[Problem Independent Regret Bound] \label{theo:main_regret_bound:sec:results}
For any $\epsilon,\gamma > 0$, choose the radius of the ellipsoid in
Algorithm~\ref{alg:gen_struct_bandit:sec:alg} as
\bea
\beta = c_0 \psi(E_{r,\max}) \left( w(\Omega_R) + \sqrt{\gamma^2 + \log{T}}\frac{\phi(\Omega_R)}{2} \right)~.
\eea
Then, for any $T > c' w^2(A_{\max})(\epsilon^2+\log{T})$, with probability at
least $1 - c_1 \exp(-w^2(A_{\max})\epsilon^2)-c_2\exp(-\gamma^2)$, the fixed
cumulative regret of Algorithm~\ref{alg:gen_struct_bandit:sec:alg} is at most
\bea
R_T \leq O\left( \psi(E_{r,\max}) \left( w(\Omega_R) + \sqrt{\gamma^2 + \log{T}} \right) \sqrt{p} \sqrt{T\log{T}} \right)~,
\eea
where $c', c_0, c_1, c_2 > 0$ are constants.
\end{theo}

%
%

\begin{theo}[Problem Dependent Regret Bound] \label{theo:main_regret_bound_dep:sec:results}
For any $\epsilon,\gamma > 0$, choose the radius of the ellipsoid in
Algorithm~\ref{alg:gen_struct_bandit:sec:alg} as
\bea
\beta = c_0 \psi(E_{r,\max}) \left( w(\Omega_R) + \sqrt{\gamma^2 + \log{T}}\frac{\phi(\Omega_R)}{2} \right)~.
\eea
Then, for any $T > c' w^2(A_{\max})(\epsilon^2+\log{T})$, with probability at
least $1 - c_1 \exp(-w^2(A_{\max})\epsilon^2)-c_2\exp(-\gamma^2)$, the fixed
cumulative regret of Algorithm~\ref{alg:gen_struct_bandit:sec:alg} with a decision
set which has non-zero gap $\Delta>0$ is at most
\bea
R_T \leq O\left( \psi^2(E_{r,\max}) \left( w(\Omega_R) + \sqrt{\gamma^2 + \log{T}} \right)^2 p \log{T} / \Delta \right)~,
\eea
where $c', c_0, c_1, c_2 > 0$ are constants.
\end{theo}

%% file: examples.tex
\subsection{Examples}
We present the problem independent regret of popular types of structured $\theta^*$
using Theorem~\ref{theo:main_regret_bound:sec:results} and the values of
$\psi(E_{r,\max})$ and $w(\Omega_R)$ from~\citep{crpw12,bcfs15,chba15b}.
The problem dependent regret can be similarly computed. Only unstructured
and sparse structures have been considered~\cite{dahk08,abps11,abps12,camu12}.
No previous works have considered any other types of structure including group sparse and low-rank.

\begin{exm}[Unstructured]  \label{exm:ridge:sec:examples}
For problems where $\theta^*$ is not structured, we simply use
$R(\theta) = \|\theta\|_2^2$ and solve the ridge regression problem
\begin{align} \label{eq:ridge:sec:examples}
\hat{\theta}_t &= \underset{\theta \in \mathbb{R}^p}{\argmin} \; \frac{1}{t} \|y_t-X_t\theta\|_2^2 + \lambda_t \|\theta\|_2^2 
= \left(X_t^\top X_t + \lambda_t \mathbb{I}_{p \times p}\right)^{-1} X_t^\top y_t.
\end{align}

We compute the regret by plugging in the values $\psi(E_{r,\max}) = \bigO(1)$ and
$w(\Omega_R) = \bigO(\sqrt{p})$ to obtain a regret of
$\widetilde{\bigO}(p \sqrt{T})$. Such a regret
matches~\citep{dahk08,abps11} up to log and constant factors.
\end{exm}

\begin{exm}[Sparse] \label{exm:lasso:sec:examples}
For problems where $\theta^*$ is $s$-sparse ($s$ non-zeros), one common regularizer to induce
sparse solutions is $R(\theta) = \|\theta\|_1$. With such a regularizer, we
solve the Lasso problem
\begin{equation} \label{eq:lasso:sec:examples}
\hat{\theta}_t = \underset{\theta \in \mathbb{R}^p}{\argmin} \; \frac{1}{t} \|y_t-X_t\theta\|_2^2 + \lambda_t \|\theta\|_1~.
\end{equation}

We compute the regret by plugging in the values
$\psi(E_{r,\max}) = \bigO(\sqrt{s})$ and $w(\Omega_R) = \bigO(\sqrt{\log{p}})$ to
obtain a regret of $\widetilde{\bigO}(\sqrt{s\log{p}} \sqrt{p} \sqrt{T})$ which
matches~\citep{abps12} up to log and constant factors. Note, it is worse than
the regret from~\citep{camu12} which is
$\bigOtilde(s\sqrt{T})$ however, they consider a different noise model in
the loss function.
\end{exm}

\begin{exm}[Group Sparse] \label{exm:group_lasso:sec:examples}
Let $\{1,\dots,p\}$ be an index set of $\theta^*$,
$\mathcal{G} = \{\mathcal{G}_1, \dots, \mathcal{G}_K\}$ be a known set of $K$
groups which define a disjoint partitioning of the index set.
For group sparse problems, one common regularizer is
$R(\theta) = \sum_{i=1}^K \|\theta_{\mathcal{G}_i}\|_2$
where $\theta_{\mathcal{G}_i}$ is a vector with elements equal to $\theta$ for
indices in $\mathcal{G}_i$ and 0 otherwise. With such a regularizer, we
solve the group lasso problem
\begin{equation} \label{eq:group_lasso:sec:examples}
\hat{\theta}_t = \underset{\theta \in \mathbb{R}^p}{\argmin} \; \frac{1}{t} \|y_t-X_t\theta\|_2^2 + \lambda_t \sum_{i=1}^K \|\theta_{\mathcal{G}_i}\|_2~.
\end{equation}

With maximum group size $m = \max_i |\mathcal{G}_i|$ and subset
$\mathcal{S}_\mathcal{G} \subset \{1, \dots, K\}$ of the groups with cardinality
$s_\mathcal{G}$ which denotes the number of active groups, we compute the regret by
plugging in the values $\psi(E_{r,\max}) = \bigO(\sqrt{s_\mathcal{G}})$ and
$w(\Omega_R) = \bigO(\sqrt{m+\log{K}})$ to obtain a regret of
$\widetilde{\bigO}(\sqrt{s_\mathcal{G}(m + \log{K})} \sqrt{p}\sqrt{T})$.
\end{exm}

\begin{exm}[Low-Rank] \label{exm:low_rank:sec:examples}
Let $\Theta^* \in \mathbb{R}^{d \times p}$ be a matrix with rank $r$ and we
select the matrix $X_t \in \mathbb{R}^{d \times p}$ at each round. Define the loss we
receive as $\ell_t(X_t) = \trace(X_t^\top\Theta^*) + \eta_t$. For problems where
the rank of $\Theta^*$ is small, for example, $r \leq \min(d,p)$, one common
regularizer to use is the nuclear norm
$R(\Theta) = \|\Theta\|_* = \sum_{j=1}^{\min\{d,p\}} \sigma_j(\Theta)$ where
$\sigma_j(\Theta)$ are the singular values of the $\Theta$. With such a regularizer,
we solve the trace-norm regularized least squares problem
\begin{equation} \label{eq:low_rank:sec:examples}
\hat{\Theta}_t = \underset{\Theta \in \mathbb{R}^{d \times p}}{\argmin} \; \frac{1}{t} \sum_{i=1}^t \left(y_i-\trace(X_i^\top\Theta)\right)^2 + \lambda_t \|\Theta\|_*~.
\end{equation}

We compute the regret by plugging in the values
$\psi(E_{r,\max}) = \bigO(\sqrt{r})$ and $w(\Omega_R) = \bigO(\sqrt{d+p})$ from~\citep{newa11}
to obtain a regret of $\widetilde{\bigO}(\sqrt{r(d + p)} \sqrt{p} \sqrt{T})$.
\end{exm}

%% file: analysis.tex
The analysis starts from a regret result established in~\cite{dahk08}.
Let $r_t {=} \langle x_t, \theta^* \rangle - \langle x^*, \theta^* \rangle$ denote
the instantaneous regret acquired by the algorithm on round $t$
where $x^* {=} \argmin_{x \in \mathcal{X}} \langle x, \theta^* \rangle$ is the
optimal vector.
Then for Algorithm~\ref{alg:gen_struct_bandit:sec:alg}, as long as we have
$\theta^* {\in} C_t$ over all rounds $t$,~\cite[Theorem 6]{dahk08} shows that
$\sum_{t=1}^T r_t^2 \leq 8\beta^2 p \log{T}$.
Then, to establish a problem independent regret bound
we directly apply the Cauchy-Schwarz inequality to get
\begin{align} \label{eq:regret_bound_cs:eq:analysis}
R_T &= \sum_{t=1}^T r_t \leq \left( T \sum_{t=1}^T r_t^2 \right)^{1/2} \leq \beta \sqrt{8pT\log{T}}~,
\end{align}
which holds conditioned on $\theta^* \in C_t$ over all rounds $t$.
Moreover, for a problem dependent regret bound, we follow the proof
of~\cite[Theorem 1]{dahk08} which shows
\begin{align} \label{eq:regret_bound_cs_dep:eq:analysis}
R_T = \sum_{t=1}^T r_t \leq \sum_{t=1}^T \frac{r_t^2}{\Delta} \leq \frac{8p\beta^2 \log{T}}{\Delta}~,
\end{align}
which holds conditioned on $\theta^* \in C_t$ over all rounds $t$.

The focus of our analysis is then to choose a $\beta$ such that the condition
holds with high-probability uniformly over all rounds.
From Algorithm~\ref{alg:gen_struct_bandit:sec:alg}, since
$C_t := \{ \theta : \| \theta - \hat{\theta}_t \|_{2,D_t} \leq \beta \}$ and we
want to have $\theta^* \in C_t$, we focus on bounds for
$\| \hat{\theta}_t - \theta^* \|_{2,D_t}$, the instantaneous estimation error.
Building on ideas for high-dimensional structured estimation as discussed in
Section~\ref{sec:est}, deterministic bounds on the instantaneous estimation
error can be obtained under two assumptions. First, we need to choose the
regularization parameter $\lambda_t$
such that
\bea
\lambda_t \geq 2 R^*\left(\frac{1}{t} X_t^\top (y_t - X_t \theta^*) \right)~.
\label{eq:asm1}
\eea
Second, for all $\hat{\theta}_t-\theta^* \in E_{r,t}$, we need to have the
restricted eigenvalue (RE) condition for constant $\kappa > 0$
\bea
\inf_{\hat{\theta}_t-\theta^* \in E_{r,t}} \frac{1}{t} \| X_t (\hat{\theta}_t-\theta^*) \|_2^2 \geq \kappa \| \hat{\theta}_t-\theta^* \|_2^2~.
\label{eq:asm2}
\eea
Under these two assumptions, following existing analysis for high-dimensional
estimation, we have the following theorem (refer to
Section~\ref{sec:appendix_ellipsoid} for the proof).
\begin{theo} \label{theo:ellipsoid_bound:sec:analysis}
Assume that the RE condition is satisfied in the set $E_{r,t}$ with parameter $\kappa$
and $\lambda_t$ is suitably large. Then for any norm $R(\cdot)$, we have for
constant $c>0$
\bea \label{eq:ellipsoidal_bound:sec:analysis}
\| \hat{\theta}_t - \theta^* \|_{2,D_t} \leq c \psi(E_{r,t}) \frac{\lambda_t}{\kappa} \sqrt{t}~.
\eea
\end{theo}
In the Sections~\ref{sec:appendix_lambda_t} and~\ref{sec:appendix_re}, we show
that the two assumptions in fact hold with high-probability. In
particular, for the assumption in \eqref{eq:asm1}, we show the following result.
\begin{theo} \label{theo:lambda_t:sec:analysis}
For any $\gamma > 0$ and for absolute constant $L>0$, with probability at
least $1 - L \exp(-\gamma^2)$, the following bound holds 
uniformly for all rounds $t=1,\dots,T$:
\begin{equation}
R^*\left(\frac{1}{t} X_t^\top (y_t - X_t \theta^*) \right) \leq 2LKB \frac{\left( w(\Omega_R) + \sqrt{\gamma^2 + \log{T}} \frac{\phi(\Omega_R)}{2} \right)}{\sqrt{t}}~.
\end{equation}
\end{theo}
Then, from \eqref{eq:asm1}, for $c_0 = 4LKB$ we set $\lambda_t$ as
\bea \label{eq:lambda_val:sec:analysis}
\lambda_t \geq c_0 \frac{\left( w(\Omega_R) + \sqrt{\gamma^2+\log{T}}\frac{\phi(\Omega_R)}{2} \right)}{\sqrt{t}}~.
\eea

As for the assumption in \eqref{eq:asm2}, we show the following result.
\begin{theo} \label{theo:re:sec:analysis}
For constants $c_0,c_1,c_2,c_3,c_4,c_5,c_6,c_7>0$ and any $\epsilon > 0$, with
probability at least $1-c_0\exp(-w^2(A_{\max})\epsilon^2)$ the following holds
uniformly for all rounds $t=1,\dots,T$:
\begin{equation} \label{eq:re_bound:sec:re}
\inf_{\hat{\theta}_t-\theta^* \in E_{r,t}} \frac{1}{t} \|X_t (\hat{\theta}_t-\theta^*)\|_2^2 
\geq c_1\left(1-c_2\frac{w(A_{\max})\sqrt{c_3\epsilon^2 + c_4\log{T}}}{\sqrt{t}}\right)
- c_5\frac{w(A_{\max})\left(2+\sqrt{c_6\epsilon^2+c_7\log{T}}\right)}{\sqrt{t}}~. \notag
\end{equation}
After $t \geq c'w^2(A_{\max})(\epsilon^2+\log{T})$, the quantity will be positive for some constant $c'$.
\end{theo}

Theorem~\ref{theo:re:sec:analysis} shows that after round $t$ crosses a
suitably scaled version of $w^2(A_{\max})$ then there exists a constant $\kappa$
such that the RE assumption holds with high-probability. Note that this
requirement implies a phase shift at which point the estimator starts to work
and forms the basis of sampling the arms i.i.d. for the initial set of
rounds in Algorithm~\ref{alg:gen_struct_bandit:sec:alg}.

For a bound on the instantaneous ellipsoidal estimation error, we plug in the
value of $\lambda_t$ from (\ref{eq:lambda_val:sec:analysis})
into (\ref{eq:ellipsoidal_bound:sec:analysis}) and use the norm compatibility
constant of the largest restricted error set to obtain
\begin{equation*}
\|\hat{\theta}_t-\theta^*\|_{2,D_t} \leq C \psi(E_{r,\max}) \left( w(\Omega_R) + \sqrt{\gamma^2+\log{T}}\frac{\phi(\Omega_R)}{2} \right)
\end{equation*}
where $C = c_0c/\kappa$ is a constant which holds with high-probability across
all rounds $t=1,\dots,T$. Therefore, if we set
\begin{equation} \label{eq:beta_val:sec:analysis}
\beta = C \psi(E_{r,\max}) \left( w(\Omega_R) + \sqrt{\gamma^2+\log{T}}\frac{\phi(\Omega_R)}{2} \right)
\end{equation}
the confidence ellipsoid $C_t$ will contain $\theta^*$ across all rounds
with high-probability. Substituting our $\beta$ into the
regret bounds in (\ref{eq:regret_bound_cs:eq:analysis}) and
(\ref{eq:regret_bound_cs_dep:eq:analysis}) gives our main result in
Theorem~\ref{theo:main_regret_bound:sec:results} and
Theorem~\ref{theo:main_regret_bound_dep:sec:results}.

%% file: conc.tex
We studied the stochastic linear bandit problem under structural assumptions on
$\theta^*$ and focused on constructing confidence ellipsoids which contain the unknown
parameter $\theta^*$ across all rounds with high-probability. We showed how to
construct such confidence ellipsoids which are general enough to hold for any
norm structured $\theta^*$ and are tighter than previous works leading to
sharper regret bounds
which, for the problem independent regret scales as
$\bigOtilde(\psi(E_{r,\max}) w(\Omega_R) \sqrt{p} \sqrt{T})$ and for the problem dependent
regret scales as $\bigOtilde(\psi^2(E_{r,\max}) w^2(\Omega_R) p \log{T}/\Delta)$.
For unstructured and $s$-sparse $\theta^*$, such regret bounds match existing
results on the standard stochastic linear bandit problem. For all other types
of structured $\theta^*$ including group
sparse and low-rank, the bounds are sharper.
%

\vspace*{3mm}
{\bf Acknowledgements:} The research was supported by NSF grants IIS-1447566, IIS-1422557, CCF-1451986,
CNS-1314560, IIS-0953274, IIS-1029711, and by NASA grant NNX12AQ39A. The authors
also acknowledge support from Adobe, IBM, and Yahoo.
%

%% file: appendix.tex
\section{Definitions and Background} \label{sec:appendix_defns}
\input{appendix_defns}

\section{Ellipsoid Bound} \label{sec:appendix_ellipsoid}
\input{appendix_ellipsoid}

\section{Algorithm} \label{sec:appendix_alg}
\input{appendix_alg}

\section{Bound on Regularization Parameter $\lambda_t$} \label{sec:appendix_lambda_t}
\input{appendix_lambda_t}

\section{Restricted Eigenvalue (RE) Condition} \label{sec:appendix_re}
\input{appendix_re}

\subsection{Anisotropic Sub-Gaussian Design Extension to Martingales} \label{ssec:appendix_mend}
\input{appendix_mend}


%% file: appendix_defns.tex
The following definitions and lemmas can be found in~\cite{bcfs14,bcfs15,vers12}.

\begin{defn} \label{defn:sub-gaussian:sec:appendix}
A random variable $x$ is sub-Gaussian if the moments satisfies
\begin{equation}
\left[ \mathbb{E} |x|^p \right]^{\frac{1}{p}} \leq K \sqrt{p}
\end{equation}
for any $p \geq 1$ with constant $K$. The minimum value of $K$ is called the
sub-Gaussian norm of $x$ and denoted by $\subgnorm{x}$.

Additionally, every sub-Gaussian random variable satisfies
\begin{equation} \label{eq:subg_prob:sec:appendix}
P(|x| > t) \leq \exp\left(1-c\frac{t^2}{\|x\|_{\psi_2}^2}\right)
\end{equation}
for all $t \geq 0$.
\end{defn}

\begin{defn} \label{defn:subg_norm:sec:appendix}
A random vector $X \in \mathbb{R}^p$ is sub-Gaussian if the one-dimensional
marginals $\langle X, x \rangle$ are sub-Gaussian random variables for all
$x \in \mathbb{R}^p$. The sub-Gaussian norm of $X$ is defined as
\begin{equation}
\subgnorm{X} = \underset{x \in S^{p-1}}{\sup} \; \|\langle X, x \rangle \|_{\psi_2}
\end{equation}
\end{defn}

\begin{defn} \label{defn:subg_width:sec:appendix}
For any set $A \in \mathbb{R}^p$, the Gaussian width of the set $A$ is defined as
\begin{equation}
w(A) = \mathbb{E}\left[ \sup_{u \in A} \langle g, u \rangle \right]
\end{equation}
where the expectation is over $g \sim N(0,\mathbb{I}_{p \times p})$ which is a
vector of independent zero-mean unit-variance Gaussian random variables.
\end{defn}

\begin{lemm} \label{lemm:subg_bounded:sec:appendix}
For any bounded random variable $|X| \leq B$, then $X$ is a sub-Gaussian random
variable with $\|X\|_{\psi_2} \leq B$.
\end{lemm}

\begin{lemm} \label{lemm:subg_vector:sec:appendix}
Consider a sub-Gaussian random vector $X$ with sub-Gaussian norm
$K = \max_i \subgnorm{X_i}$, then, for vector $a$, $Z = \langle X, a \rangle$
is a sub-Gaussian random variable with sub-Gaussian norm $\subgnorm{Z} \leq CK\|a\|_2$
for absolute constant $C$.
\end{lemm}

%% file: appendix_ellipsoid.tex
\begin{theorepeat}{Theorem~\ref{theo:ellipsoid_bound:sec:analysis}}
Assume that the RE condition is satisfied in the set $E_{r,t}$ with parameter $\kappa$
and $\lambda_t$ is suitably large. Then for any norm $R(\cdot)$ we have for
constant $c>0$
\bea
\| \hat{\theta}_t - \theta^* \|_{2,D_t} \leq c \psi(E_{r,t}) \frac{\lambda_t}{\kappa} \sqrt{t}~.
\eea
\end{theorepeat}

\begin{proof} Proof of Theorem~\ref{theo:ellipsoid_bound:sec:analysis}.

For any $\hat{\theta}_t-\theta^*_t \in E_{r,t}$ and by the definition of a convex function
\begin{equation*}
\mathcal{L}(\hat{\theta}_t) - \mathcal{L}(\theta^*) \geq \langle \nabla \mathcal{L}(\theta^*), \hat{\theta}_t-\theta^* \rangle~.
\end{equation*}
Moreover, by the definition of a dual norm we have
\begin{equation*}
|\langle \nabla \mathcal{L}(\theta^*), \hat{\theta}_t-\theta^* \rangle| \leq R^*(\nabla \mathcal{L}(\theta^*)) R(\hat{\theta}_t-\theta^*)~.
\end{equation*}
By construction following (\ref{eq:lambda_n:sec:est}) from Section~\ref{sec:est},
for any $\rho > 0$ (not just $\rho = 2$) we get
\begin{equation*}
R^*(\nabla \mathcal{L}(\theta^*)) \leq \frac{\lambda_t}{\rho}
\end{equation*}
which implies
\begin{align*}
|\langle \nabla \mathcal{L}(\theta^*), \hat{\theta}_t-\theta^* \rangle| &\leq \frac{\lambda_t}{\rho}  R(\hat{\theta}_t-\theta^*) \\
\Rightarrow \langle \nabla \mathcal{L}(\theta^*), \hat{\theta}_t-\theta^* \rangle &\geq -\frac{\lambda_t}{\rho}  R(\hat{\theta}_t-\theta^*)~.
\end{align*}
Therefore,
\begin{align*}
\mathcal{L}(\hat{\theta}_t) - \mathcal{L}(\theta^*) &\geq -\frac{\lambda_t}{\rho} R(\hat{\theta}_t-\theta^*) \\
\Rightarrow |\mathcal{L}(\hat{\theta}_t) - \mathcal{L}(\theta^*)| &\leq \frac{\lambda_t}{\rho} R(\hat{\theta}_t-\theta^*)~.
\end{align*}
By the definition of the norm compatibility constant $\psi(E_{r,t}) = \sup_{u \in E_{r,t}} \frac{R(u)}{\|u\|_2}$
we have $R(\hat{\theta}_t-\theta^*) \leq \|\hat{\theta}_t-\theta^*\|_2 \psi(E_{r,t})$ which implies 
\begin{equation*}
|\mathcal{L}(\hat{\theta}_t) - \mathcal{L}(\theta^*)| \leq \frac{\lambda_t}{\rho} \|\hat{\theta}_t-\theta^*\|_2 \psi(E_{r,t})~.
\end{equation*}
Therefore, for the squared loss, since 
$\mathcal{L}(\hat{\theta}_t) - \mathcal{L}(\theta^*) = \frac{1}{t}\|X_t (\hat{\theta}_t-\theta^*)\|_2^2$
we get
\begin{align*}
|\mathcal{L}(\hat{\theta}_t) - \mathcal{L}(\theta^*)| = \left|\frac{1}{t}\|X_t (\hat{\theta}_t-\theta^*)\|_2^2\right| = \frac{1}{t} \|X_t (\hat{\theta}_t-\theta^*)\|_2^2~.
\end{align*}
Therefore,
\begin{equation*}
\frac{1}{t}\|X_t (\hat{\theta}_t-\theta^*)\|_2^2 \leq \frac{\lambda_t}{\rho} \|\hat{\theta}_t-\theta^*\|_2 \psi(E_{r,t})~.
\end{equation*}

Using the bound in (\ref{eq:delta_n:sec:est}) from Section~\ref{sec:est} for
$\|\hat{\theta}_t-\theta^*\|_2$ we obtain
\begin{equation*}
\frac{1}{t} \|X_t (\hat{\theta}_t-\theta^*)\|_2^2 \leq \frac{\lambda_t}{\rho} \psi(E_{r,t}) \frac{\lambda_t}{\kappa} \psi(E_{r,t})~.
\end{equation*}
Finally, noting
$\|X_t (\hat{\theta}_t-\theta^*)\|_2^2 = \|\hat{\theta}_t-\theta^*\|_{2,D_t}^2$, multiplying each
side by $t$, and taking the square root of both sides we get the final bound
\begin{equation*}
\|\hat{\theta}_t-\theta^*\|_{2,D_t} \leq c \psi(E_{r,t}) \frac{\lambda_t}{\kappa} \sqrt{t}
\end{equation*}
for constant $c > 0$ which ends the proof.
\end{proof}
\qed

%% file: appendix_alg.tex
In this section, we will show that selecting an $x_{t+1}$ following
Algorithm~\ref{alg:gen_struct_bandit:sec:alg}, that we can compute a
$\widetilde{\theta}_{t+1}$ such that the inequality
$\langle x_{t+1}, \widetilde{\theta}_{t+1} \rangle \leq \langle x^*, \theta^* \rangle$
holds. In Algorithm~\ref{alg:gen_struct_bandit:sec:alg}, we select $x_{t+1}$
by computing
$(x_{t+1}', \theta_{t+1}') = \argmin_{x \in \mathcal{X}, \theta \in C_t \cap S^{p-1}} \langle x, \theta \rangle$
and sampling uniformly at random from a closed $L_2$ ball centered at $x_{t+1}'$
with radius $\|x_{t+1}'\|_2/2$. In the following lemma, we show a specific way
of computing such a sample which shows why the radius
is set to such a value. Moreover, we will prove that we can deterministically
compute a $\widetilde{\theta}_{t+1}$ such that the inequality above holds.

\begin{lemm}
For a decision set $\mathcal{X}$ and a confidence ellipsoid $C_t$, if we compute
\begin{equation*}
(x_{t+1}', \theta_{t+1}') = \underset{\substack{x \in \mathcal{X} \\ \theta \in C_t \cap S^{p-1}}}{\argmin} \langle x, \theta \rangle
\end{equation*}
and set $x_{t+1}$ and $\widetilde{\theta}_{t+1}$ as
\begin{align*}
x_{t+1} &= x_{t+1}' + \xi_x v \\
\widetilde{\theta}_{t+1} &= \theta_{t+1}' + \xi_\theta u
\end{align*}
where $v$ is a random vector such that $\|v\|_2 \leq 1$ and
$x_{t+1}'+\xi_x v \in \mathcal{X}$, $u = -x_{t+1}'/\|x_{t+1}'\|_2$,
$\xi_x = \|x_{t+1}'\|_2/2$, and $\xi_\theta = 1$ then the inequality
$\langle x_{t+1}, \widetilde{\theta}_{t+1} \rangle \leq \langle x_{t+1}', \theta_{t+1}' \rangle \leq \langle x^*, \theta^* \rangle$
holds.
\end{lemm}

\begin{proof}
First, the inequality
$\langle x_{t+1}', \theta_{t+1}' \rangle \leq \langle x^*, \theta^* \rangle$
holds because we assume $\|\theta^*\|_2=1$ and the optimization is over both
$x$ and $\theta$. Second, conceptually, we sample $x_{t+1}$ from the intersection
of the decision set and a ball.
Above, we provided a specific example of how to set the radius of such a ball and
sample $x_{t+1}$ and $\widetilde{\theta}_{t+1}$.
Now, we will show that selecting
$x_{t+1}$ and $\widetilde{\theta}_{t+1}$ and setting the radii as above
satisfies the inequality
$\langle x_{t+1}, \widetilde{\theta}_{t+1} \rangle \leq \langle x_{t+1}', \theta_{t+1}' \rangle$.

First, observe
\begin{align*}
\langle x_{t+1}, \widetilde{\theta}_{t+1} \rangle &= \left\langle x_{t+1}'+\frac{\|x_{t+1}'\|_2}{2} v, \theta_{t+1}'-\frac{x_{t+1}'}{\|x_{t+1}'\|_2} \right\rangle \\
&= \langle x_{t+1}', \theta_{t+1}' \rangle - \|x_{t+1}'\|_2 + \frac{\|x_{t+1}'\|_2}{2} \langle v, \theta_{t+1}' \rangle - \frac{1}{2} \langle v, x_{t+1}' \rangle~.\\
\end{align*}
We need to show that
\begin{align}
&\langle x_{t+1}', \theta_{t+1}' \rangle - \|x_{t+1}'\|_2 + \frac{\|x_{t+1}'\|_2}{2} \langle v, \theta_{t+1}' \rangle - \frac{1}{2} \langle v, x_{t+1}' \rangle \leq \langle x_{t+1}', \theta_{t+1}' \rangle \notag \\
&\Rightarrow \frac{\|x_{t+1}'\|_2}{2} \langle v, \theta_{t+1}' \rangle \leq \|x_{t+1}'\|_2 + \frac{1}{2} \langle v, x_{t+1}' \rangle \notag \\
&\Rightarrow -\frac{\|x_{t+1}'\|_2}{2} \leq \frac{1}{2} \langle v, x_{t+1}' \rangle \quad \text{(since $|\langle v, \theta_{t+1}'\rangle| \leq 1$)} \notag \\
&\Rightarrow 0 \leq \langle v, x_{t+1}' \rangle + \|x_{t+1}'\|_2~.\label{eq:alg_ball_eq:sec:appendix_alg}
\end{align}
From the Cauchy-Schwarz inequality we have
\begin{align*}
&|\langle v, x_{t+1}' \rangle| \leq \|v\|_2 \|x_{t+1}'\|_2 \\
\Rightarrow &\langle v, x_{t+1}' \rangle \geq -\|v\|_2 \|x_{t+1}'\|_2 \\
\Rightarrow &\langle v, x_{t+1}' \rangle \geq - \|x_{t+1}'\|_2 \quad \text{(since $v$ is a unit vector)} \\
\end{align*}
\end{proof}
Plugging this in (\ref{eq:alg_ball_eq:sec:appendix_alg}) completes the proof.
\qed

%% file: appendix_lambda_t.tex
We will prove the following main theorem.

\begin{theorepeat}{Theorem~\ref{theo:lambda_t:sec:analysis}}
For any $\gamma > 0$ and for absolute constant $L>0$, with probability at least
$1-L\exp(-\gamma^2)$, the following bound holds uniformly for all
$t=1,\dots,T$:
\begin{equation}
R^*\left(\frac{1}{t} X_t^\top (y_t-X_t\theta^*)\right) \leq 2LKB\frac{\left(w(\Omega_R)+\sqrt{\gamma^2+\log{T}}\frac{\phi(\Omega_R)}{2}\right)}{\sqrt{t}}~.
\end{equation}
\end{theorepeat}

\begin{proof} Proof of Theorem~\ref{theo:lambda_t:sec:analysis}.

Recall the regularization parameter $\lambda_t$ needs to satisfy the inequality
\begin{equation} \label{eq:lambda_t_inequality:sec:analysis}
\lambda_t \geq \rho R^*(\nabla \mathcal{L}(\theta^*, Z_t)) = \rho R^*\left(\frac{1}{t}X_t^\top(y_t-X_t\theta^*)\right)
\end{equation}
for $\rho > 1$.
Two issues of the right hand side are (1) the expression depends on the unknown
parameter $\theta^*$ and (2) the expression is a random variable since it depends
on $n$ vectors selected uniformly at random from the decision set $\mathcal{X}$
and a sequence of random
noise terms $\eta_1, \dots, \eta_t$. We can remove the dependence on $\theta^*$
by observing that $y_t-X_t\theta^*$ is precisely the $t$-dimensional noise vector
$\omega_t = [\eta_1 \dots \eta_t]^\top$. Therefore,
\begin{equation}
R^*\left(\frac{1}{t}X_t^\top(y_t-X_t\theta^*)\right) = R^*\left(\frac{1}{t}X_t^\top \omega_t\right)~. 
\end{equation}

By the definition of the dual norm
$R^* \left( \frac{1}{t} X_t^\top \omega_t \right) =
\sup_{R(u) \leq 1} \frac{1}{t} \left\langle X_t^\top \omega_t, u \right\rangle$.
The proof involves showing that
$\frac{1}{t} \langle X_t^\top \omega_t, u \rangle$ is a martingale difference
sequence (MDS) which concentrates as a sub-Gaussian random variable. Then, using
a generic chaining argument, we show the supremum of such a quantity also
concentrates as a sub-Gaussian random variable. 

We begin by observing that
\begin{equation}
\frac{1}{t} \langle X_t^\top \omega_t, u \rangle = \frac{1}{\sqrt{t}} \frac{1}{\sqrt{t}} \langle X_t^\top \omega_t, u \rangle~.
\end{equation}
\noindent We will save one of the $\frac{1}{\sqrt{t}}$ terms for later and now
proceed to show how $\frac{1}{\sqrt{t}}\left\langle X_t^\top \omega_t, u \right\rangle$
concentrates.

\subsection{$\frac{1}{\sqrt{t}}\left\langle X_t^\top \omega_t, u \right\rangle$ Concentrates as a Sub-Gaussian}
First, let
\begin{equation}
\frac{1}{\sqrt{t}}\left\langle X_t^\top \omega_t, u \right\rangle
= \|u\|_2 \frac{1}{\sqrt{t}}\left\langle X_t^\top \omega_t, \frac{u}{\|u\|_2} \right\rangle
= \|u\|_2 \frac{1}{\sqrt{t}}\left\langle X_t^\top \omega_t, q \right\rangle
\end{equation}
where $q = \frac{u}{\|u\|_2}$. We focus on the term 
$\frac{1}{\sqrt{t}}\left\langle X_t^\top \omega_t, q \right\rangle$.
We can construct a martingale difference sequence (MDS) by observing that
\begin{align}
\left\langle X_t^\top \omega_t, q \right\rangle = \left\langle \omega_t, X_t q \right\rangle = \sum_{\tau=1}^t \eta_\tau \langle x_\tau, q \rangle = \sum_{\tau=1}^t z_\tau
\end{align}
for $z_\tau = \eta_t \langle x_\tau, q \rangle$. Recall from
Assumption~\ref{asmp:noise:sec:asmps_defns} the filtration is defined as
$F_t = \{x_1, \dots, x_{t+1}, \eta_1, \dots, \eta_t\}$.
Each $z_\tau$ can be seen as a MDS since
\begin{equation}
\mathbb{E}[z_\tau | F_{\tau-1}]
= \mathbb{E}[\eta_\tau \langle x_\tau, q \rangle | F_{\tau-1}]
= \langle x_\tau, q \rangle \cdot \mathbb{E}[\eta_\tau | F_{\tau-1}] = 0
\end{equation}
because $x_\tau$ is $F_{\tau-1}$ measurable and $\eta_\tau$ is
$F_\tau$
measurable. Additionally, each $z_\tau$ follows a sub-Gaussian distribution
with parameter $KB$ because
$\subgnorm{\eta_\tau \langle x_\tau, q \rangle} \leq KB$
(Assumption~\ref{asmp:noise:sec:asmps_defns} and
Definition~\ref{defn:subg_norm_constant:sec:asmps_defns}).
Since each $z_\tau$ is a bounded MDS, we can use the Azuma-Hoeffding
inequality to show that the sum $\sum_{\tau=1}^t z_\tau$ concentrates as a
sub-Gaussian with parameter $KB$. For all $\gamma \geq 0$
\begin{align} \label{eq:mds_concentration_unit:sec:analysis}
P \left( \left| \sum_{\tau=1}^t z_\tau \right| \geq \gamma \right) &= P \left(\left| \langle X_t^\top \omega_t, q \rangle \right| \geq \gamma \right) \notag \\
&= P \left(\left| \left\langle X_t^\top \omega_t, \frac{u}{\|u\|_2} \right\rangle \right| \geq \gamma \right) \leq 2\exp \left( \frac{-\gamma^2}{2tK^2B^2} \right) \notag \\
&= P \left(\frac{1}{\sqrt{t}} \left| \left\langle X_t^\top \omega_t, \frac{u}{\|u\|_2} \right\rangle \right| \geq \zeta \right) \leq 2\exp \left( \frac{-\zeta^2}{2K^2B^2} \right)
\end{align}
where $\zeta = \gamma/\sqrt{t}$ which implies $\gamma = \sqrt{t} \zeta$.
From (\ref{eq:mds_concentration_unit:sec:analysis})
and (\ref{eq:subg_prob:sec:appendix}) in
Definition~\ref{defn:sub-gaussian:sec:appendix} (Section~\ref{sec:appendix_defns})
we can see that the term
$\frac{1}{\sqrt{t}}\left\langle X_t^\top \omega_t, \frac{u}{\|u\|_2} \right\rangle$ concentrations
as a sub-Gaussian with
$\subgnorm{\langle X_t^\top \omega_t, \frac{u}{\|u\|_2} \rangle} \leq KB$.

Next, we show that the term 
$\frac{1}{\sqrt{t}}\left\langle X_t^\top \omega_t, u \right\rangle$
also concentrations as a sub-Gaussian with
$\subgnorm{\langle X_t^\top \omega_t, u \rangle} \leq \|u\|_2KB$ using
(\ref{eq:mds_concentration_unit:sec:analysis}) as
\begin{align} \label{eq:mds_concentration:sec:analysis}
&P \left(\frac{1}{\sqrt{t}} \left| \left\langle X_t^\top \omega_t, \frac{u}{\|u\|_2} \right\rangle \right| \geq \zeta \right) \notag \\
=&P \left( \|u\|_2 \frac{1}{\sqrt{t}}\left| \left\langle X_t^\top \omega_t, \frac{u}{\|u\|_2} \right\rangle \right| \geq \|u\|_2\zeta \right) \notag \\
=&P \left( \frac{1}{\sqrt{t}}\left| \left\langle X_t^\top \omega_t, u \right\rangle \right| \geq \epsilon \right) \leq 2\exp \left( \frac{-\epsilon^2}{2\|u\|_2^2K^2B^2} \right)
\end{align}
where $\epsilon = \|u\|_2 \zeta$ which implies $\zeta = \epsilon / \|u\|_2$. 
The reason we went through showing the above is because the generic chaining
argument we will invoke to bound
$\sup_{R(u) \leq 1} \frac{1}{\sqrt{t}} \left\langle X_t^\top \omega_t, u \right\rangle$
requires that 
$\frac{1}{\sqrt{t}} \left\langle X_t^\top \omega_t, u \right\rangle$
is a sub-Gaussian random variable.

\subsection{Bound on $\sup_{R(u) \leq 1} \frac{1}{\sqrt{t}} \langle X_t^\top \omega_t, u \rangle$ via Generic Chaining}
We obtain a high-probability bound on
$\sup_{R(u) \leq 1} \frac{1}{\sqrt{t}} \langle X_t^\top \omega_t, u \rangle$
using a generic chaining argument from~\cite{tala05,tala14}. This involves
(1) showing that the absolute difference of two sub-Gaussian processes
concentrates as a sub-Gaussian, (2) showing the expectation over the supremum
of the absolute difference of two sub-Gaussian processes is upper bounded by
the sub-Gaussian width of a set from which the processes are indexed from, and
(3) showing the supremum of a sub-Gaussian process is concentrated around its
expectation and therefore, around the sub-Gaussian width with high-probability.

\medskip
\noindent \textbf{(1) Sub-Gaussian Process Concentration}
\medskip

\noindent First, we show that the absolute difference of two sub-Gaussian processes
concentrates as a sub-Gaussian.
Let
$Y_u = \frac{1}{\sqrt{t}} \langle X_t^\top \omega_t, u \rangle$ indexed by $u \in \Omega_R$ and 
$Y_v = \frac{1}{\sqrt{t}} \langle X_t^\top \omega_t, v \rangle$ indexed by $v \in \Omega_R$ be
two zero-mean (since they are both a MDS sum), random symmetric processes (since
$(Y_u)_{u \in \Omega_r}$ has the same law as $(-Y_u)_{u \in \Omega_R}$
via (\ref{eq:mds_concentration_unit:sec:analysis}) and $\omega_t$ is symmetric
and similarly for $Y_v$).
Then by construction
\begin{align*}
|Y_u-Y_v| &= \frac{1}{\sqrt{t}}\left|\left\langle X_t^\top \omega_t, u-v \right\rangle \right|~.
\end{align*}
Using the bound we established in (\ref{eq:mds_concentration:sec:analysis}), we
obtain the following bound on the absolute difference of two sub-Gaussian
random processes $Y_u$ and $Y_v$ as
\begin{align} \label{eq:process:sec:analysis}
P\left( \frac{1}{\sqrt{t}}\left|\left\langle X_t^\top \omega_t, u-v \right\rangle \right| \geq \epsilon \right) \leq 2 \exp\left( \frac{-\epsilon^2}{2\|u-v\|_2^2K^2B^2} \right)
\end{align}
which shows $|Y_u-Y_v|$ concentrates as a sub-Gaussian random variable with
$\subgnorm{Y_u-Y_v} = \|u-v\|_2KB$.

\medskip
\noindent \textbf{(2) Bound on $\mathbb{E} \left[ \sup_{R(u) \leq 1} \frac{1}{t} \langle X_t^\top \omega_t, u \rangle \right]$}
\medskip

\noindent In order to establish a high-probability bound on
$\sup_{R(u) \leq 1} \frac{1}{t} \langle X_t^\top \omega_t, u \rangle$
we need to prove a bound on
$\mathbb{E} \left[ \sup_{R(u) \leq 1} \frac{1}{t} \langle X_t^\top \omega_t, u \rangle \right]$.
To prove such a bound, we will apply a generic chaining argument for upper
bounds on such sub-Gaussian processes. For the generic chaining argument, we
will need the result in (\ref{eq:process:sec:analysis}) 
and the following lemma.


\begin{lemm} \textbf{(\cite{tala05}, Theorem 2.1.5)} \label{theo:process_exp_bound:sec:analysis}
Consider two processes $(Y_u)_{u \in \Omega_R}$ and $(X_u)_{u \in \Omega_R}$
indexed by the same set. Assume that the process $(X_u)_{u \in \Omega_R}$ is
Gaussian and that the process $(Y_u)_{u \in \Omega_R}$ satisfies the condition
\begin{equation} \label{eq:increment_condition:sec:analysis}
\forall \epsilon > 0, \forall u,v \in \Omega_R, P(|Y_u-Y_v| \geq \epsilon) \leq 2\exp\left(-\frac{\epsilon^2}{d(u,v)^2}\right)
\end{equation}
where $d(u,v)$ is a distance function which we assume is $d(u,v) = \|u-v\|_2$
for the set $\Omega_R$.
Then we have
\begin{equation}
\mathbb{E}\left[ \sup_{u,v \in \Omega_R} |Y_u-Y_v| \right] \leq L \mathbb{E} \left[ \sup_{u \in \Omega_R} X_u \right]
\end{equation}
where $L$ is an absolute constant.
\end{lemm}

First, notice that $\mathbb{E} \left[ \sup_{u \in \Omega_R} X_u \right]$ is
exactly the Gaussian width $w(\Omega_R)$ of the set $\Omega_R$
as seen by the Definition~\ref{defn:subg_width:sec:appendix} (Section~\ref{sec:appendix_defns}).
For our purposes, we make one modification to the above lemma similar
to~\cite[Theorem 8]{bcfs15}. In (\ref{eq:process:sec:analysis}), we see that
$|Y_u-Y_v|$ concentrates as a sub-Gaussian with parameter $\|u-v\|_2KB$. To
bound the expectation of two sub-Gaussian processes, we scale the Gaussian
width by the additional term $KB$ to get

\begin{equation} \label{eq:exp_subg_bound:sec:analysis}
\mathbb{E}\left[ \sup_{u,v \in \Omega_R} |Y_u-Y_v| \right] \leq L KB \mathbb{E} \left[ \sup_{u \in \Omega_R} X_u \right] = LKB w(\Omega_R)~.
\end{equation}

This shows for two sub-Gaussian processes $Y_u$ and $Y_v$, the expectation of
the supremum of their absolute difference is upper bounded by the Gaussian width
scaled by the sub-Gaussian norm, i.e., the sub-Gaussian width.

\noindent The second result we need is the following lemma.
\begin{lemm} \textbf{(\cite{tala05}, Lemma 1.2.8)} \label{lemm:symmetric:sec:analysis}
If the process $(Y_u)_{u \in \Omega_R}$ is symmetric then
\begin{equation}
\mathbb{E}\left[ \sup_{u,v \in \Omega_R} |Y_u-Y_v| \right] = 2\mathbb{E}\left[ \sup_{u \in \Omega_R} Y_u \right]~.
\end{equation}
\end{lemm}

\noindent We know from above that our processes
$Y_u = \frac{1}{\sqrt{t}} \langle X_t^\top \omega_t, u \rangle$ and
$Y_v = \frac{1}{\sqrt{t}} \langle X_t^\top \omega_t, v \rangle$ are symmetric.
As such we get the following lemma.

\begin{lemm} \label{lemm:exp:sec:analysis}
From (\ref{eq:process:sec:analysis}) we can see that the condition of
Lemma~\ref{theo:process_exp_bound:sec:analysis} is satisfied in the sub-Gaussian
case so using Lemma~\ref{theo:process_exp_bound:sec:analysis}
and Lemma~\ref{lemm:symmetric:sec:analysis} for some absolute constant $L$ we obtain
\begin{align}
&\mathbb{E}\left[ \sup_{u,v \in \Omega_R} |Y_u-Y_v| \right] 
= 2\mathbb{E}\left[ \sup_{u \in \Omega_R} |Y_u| \right] 
\leq 2L KB w(\Omega_R) \\ \notag
&\Rightarrow 2\mathbb{E}\left[ \sup_{u \in \Omega_R} \frac{1}{\sqrt{t}} \left|\left\langle X_t^\top \omega_t, u\right\rangle \right| \right]
\leq 2L KB w(\Omega_R)~.
\end{align}
\end{lemm}


\medskip
\noindent \textbf{(3) Concentration of $\sup_{R(u) \leq 1} \frac{1}{\sqrt{t}} \left\langle X_t^\top \omega_t, u \right\rangle$}
\medskip

\noindent To complete the argument, we need the following lemma.
\begin{lemm} \textbf{(\cite{tala14}, Theorem 2.2.27)} \label{theo:2.2.27:sec:analysis}
If the process $(Y_u)$ satisfies (\ref{eq:increment_condition:sec:analysis}) or
similarly (\ref{eq:process:sec:analysis}) for the sub-Gaussian
case then for $\epsilon > 0$ one has
\begin{equation}
P \left( \sup_{u,v \in \Omega_R} |Y_u-Y_v| \geq L\big(\gamma_2(\Omega_R,d(u,v)) + \epsilon \Delta(\Omega_R)\big) \right) \leq L\exp(-\epsilon^2)~.
\end{equation}
\end{lemm}
Note, the function $\Delta(\Omega_R)=\sup_{u,v \in \Omega_R} d(u,v)$ is the
diameter of the set $\Omega_R$. For our setting, $d(u,v) = \|u-v\|_2$ so we
replace $\Delta(\Omega_R)$ with $\phi(\Omega_R)$ as detailed in
Definition~\ref{defn:diameter:sec:asmps_defns} in Section~\ref{ssec:assumptions}.
The specifics of the $\gamma_2(\cdot,\cdot)$ function are not necessary for this
work since we can bound it and simplify Lemma~\ref{theo:2.2.27:sec:analysis}
by using the following lemma.
\begin{lemm} \textbf{(\cite{tala14}, Theorem 2.4.1)} \label{theo:majorizing_measure:sec:analysis}
For some universal constant $L$ we have
\begin{equation}
\frac{1}{L} \gamma_2(\Omega_R, d(u,v)) \leq \mathbb{E}\left[\sup_{u \in \Omega_R} Y_u \right]\leq L \gamma_2(\Omega_R, d(u,v))~.
\end{equation}
\end{lemm}

\noindent Combining Lemma~\ref{theo:2.2.27:sec:analysis} with
Lemma~\ref{theo:majorizing_measure:sec:analysis}, using
Lemma~\ref{lemm:exp:sec:analysis},
and our definitions of $Y_u$ and $Y_v$
for any $\epsilon > 0$ we get
\begin{lemm} \label{theo:prob_sup:sec:analysis}
\begin{align}
P \left( \sup_{R(u) \leq 1} \frac{1}{\sqrt{t}} \left| \left\langle X_t^\top \omega_t, u \right\rangle \right| \geq 2LKBw(\Omega_R)+\epsilon \right) \leq L\exp\left(-\left(\frac{\epsilon}{LKB\phi(\Omega_R)}\right)^2\right)~.
\end{align}
\end{lemm}

\begin{proof}
Proof of Lemma~\ref{theo:prob_sup:sec:analysis}.
        \begin{align*}
        &P \left( \sup_{u,v \in \Omega_R} |Y_u-Y_v| \geq L\big(\gamma_2(\Omega_R,d(u,v)) + \zeta \Delta(\Omega_R)\big) \right) \notag \\ 
        &=P \left( \sup_{u,v \in \Omega_R} |Y_u-Y_v| \geq L\gamma_2(\Omega_R,d(u,v)) + \epsilon \right) \notag \\ 
        &\leq P \left( \sup_{u,v \in \Omega_R} |Y_u-Y_v| \geq \mathbb{E} \left[ \sup_{u,v \in \Omega_R} |Y_u - Y_v| \right] + \epsilon \right) \notag \\ 
        &= P \left( \sup_{u \in \Omega_R} |Y_u| \geq 2\mathbb{E} \left[ \sup_{u \in \Omega_R} |Y_u| \right] + \epsilon \right) \notag \\ 
        &= P \left( \sup_{R(u) \leq 1} \frac{1}{\sqrt{t}} \left| \left\langle X_t^\top \omega_t, u \right\rangle \right| \geq 2LKBw(\Omega_R) + \epsilon \right) \leq L\exp\left(-\left(\frac{\epsilon}{LKB\phi(\Omega_R)}\right)^2\right)~.
        \end{align*}
where the first line comes from the left-hand side of Lemma~\ref{theo:2.2.27:sec:analysis},
the second line comes from the fact that $\Delta(\Omega_R) \leq \gamma_2(\Omega_R, d(u,v))$ from~\cite{tala14} Definition 2.2.19,
the third line comes from Lemma~\ref{theo:majorizing_measure:sec:analysis},
the fourth line comes from Lemma~\ref{lemm:symmetric:sec:analysis},
the fifth line comes from Lemma~\ref{lemm:exp:sec:analysis}, 
and the last line follows from our construction of the process $Y_u$ and the
right-hand side of Lemma~\ref{theo:2.2.27:sec:analysis}.
\end{proof}
\qed


Dividing the other $\sqrt{t}$ through and setting $\epsilon / \sqrt{t} = \alpha 2LKBw(\Omega_R)/\sqrt{t}$
we get
\begin{lemm} \label{theo:sup_gauss_bound:sec:analysis}
\begin{equation}
P \left( R^*\left(\frac{1}{t} X_t^\top \omega_t\right) \geq 2LKB(1+\alpha)\frac{w(\Omega_R)}{\sqrt{t}} \right) \leq L\exp \left( -\left(\frac{2\alpha w(\Omega_R)}{\phi(\Omega_R)}\right)^2 \right)~. 
\end{equation}
\end{lemm}

\begin{proof}
Proof of Lemma~\ref{theo:sup_gauss_bound:sec:analysis}
\begin{align*}
P \left( R^*\left(\frac{1}{\sqrt{t}} X_t^\top \omega_t\right) \geq 2LKB w(\Omega_R) + \epsilon \right) &\leq L\exp \left( -\left(\frac{\epsilon}{LKB\phi(\Omega_R)}\right)^2 \right) \\
P \left( R^*\left(\frac{1}{t} X_t^\top \omega_t\right) \geq 2LKB \frac{w(\Omega_R)}{\sqrt{t}} + \gamma \right) &\leq L\exp \left( -\left(\frac{\sqrt{t}\gamma}{LKB \phi(\Omega_R)}\right)^2 \right) \\
P \left( R^*\left(\frac{1}{t} X_t^\top \omega_t\right) \geq 2LKB \frac{w(\Omega_R)}{\sqrt{t}} + 2LKB\alpha\frac{w(\Omega_R)}{\sqrt{t}} \right) &\leq L\exp \left( -\left(\frac{\sqrt{t}\alpha 2LKB \frac{w(\Omega_R)}{\sqrt{t}}}{LKB \phi(\Omega_R)}\right)^2 \right) \\
P \left( R^*\left(\frac{1}{t} X_t^\top \omega_t\right) \geq 2LKB(1+\alpha)\frac{w(\Omega_R)}{\sqrt{t}} \right) &\leq L\exp \left( -\left(\frac{2\alpha w(\Omega_R)}{\phi(\Omega_R)}\right)^2 \right)~. 
\end{align*}
where the first inequality is from Lemma~\ref{theo:prob_sup:sec:analysis},
the second inequality is from multiplying both sides by $\frac{1}{\sqrt{t}}$
and setting $\gamma = \frac{\epsilon}{\sqrt{t}}$,
and the third inequality is from setting $\gamma = \alpha 2LKB\frac{w(\Omega_R)}{\sqrt{t}}$.
\end{proof}
\qed

Lemma~\ref{theo:sup_gauss_bound:sec:analysis} gives a high-probability bound
on the value of $R^*\left(X_t^\top \omega_t\right)$ for round $t$ but to complete
the proof of Theorem~\ref{theo:lambda_t:sec:analysis} we need a bound which
holds simultaneously for all rounds $T$ with high-probability. 
%
%
%
To obtain such a bound, we can set
$\alpha^2 = (\gamma^2+\log{T}) \left(\frac{\phi(\Omega_R)}{2w(\Omega_R)}\right)^2$
and apply a union bound for all $t$
\begin{align*}
&\bigcup_{t=1}^T P \left( R^*\left(\frac{1}{t} X_t^\top \omega_t\right) \geq 2LKB\left(1+\sqrt{\gamma^2+\log{T}} \left(\frac{\phi(\Omega_R)}{2w(\Omega_R)}\right)\right)\frac{w(\Omega_R)}{\sqrt{t}} \right) \\
        &\leq \sum_{t=1}^T L\exp \left( -(\gamma^2+\log{T}) \left(\frac{\phi(\Omega_R)}{2w(\Omega_R)}\right)^2 \left(\frac{2 w(\Omega_R)}{\phi(\Omega_R)}\right)^2 \right) \\
        &= L\sum_{t=1}^T \exp \left( -\gamma^2 - \log{T} \right) \\
&= L\sum_{t=1}^T \exp \left( -\gamma^2 \right) \times \frac{1}{T} \\
&= L\exp \left( -\gamma^2 \right)~.
\end{align*}
Rearranging the terms ends the proof of Theorem~\ref{theo:lambda_t:sec:analysis}.
\end{proof}
\qed
%

%% file: appendix_re.tex
We will prove the following theorem.

\begin{theorepeat}{Theorem~\ref{theo:re:sec:analysis}}
For constants $c_0,c_1,c_2,c_3,c_4,c_5,c_6,c_7>0$ and any $\epsilon > 0$, with probability at least
$1-c_0\exp(-w^2(A_{\max})\epsilon^2)$ the following will hold
uniformly for all rounds $t=1,\dots,T$:
\begin{equation} \label{eq:re_bound:sec:re}
\inf_{u \in E_{r,t}} \frac{1}{t} \|X_t u\|_2^2 
\geq c_1\left(1-c_2\frac{w(A_{\max})\sqrt{c_3\epsilon^2 + c_4\log{T}}}{\sqrt{t}}\right)
- c_5\frac{w(A_{\max})\left(2+\sqrt{c_6\epsilon^2+c_7\log{T}}\right)}{\sqrt{t}}~. \notag
\end{equation}
After $t \geq c'w^2(A_{\max})(\epsilon^2+\log{T})$, the quantity will be positive for some constant $c'$.
\end{theorepeat}

\begin{proof} Proof of Theorem~\ref{theo:re:sec:analysis}.

For a design matrix $X_t$ with $t$ rows, a response vector $y_t$, and parameter
$\kappa$ the RE condition is
\begin{align} \label{eq:re:sec:analysis}
&\frac{1}{t} \|y_t-X_t\hat{\theta}_t\|_2^2 - \frac{1}{t}\|y_t-X_t\theta^*\|_2^2 - \frac{1}{t} \left\langle X_t^\top(y_t-X_t\theta^*),\hat{\theta}_t-\theta^* \right\rangle \geq \kappa \|\hat{\theta}_t-\theta^*\|_2^2 \notag \\
&\Rightarrow \frac{1}{t}\|X_t(\hat{\theta}_t-\theta^*)\|_2^2 \geq \kappa \|\hat{\theta}_t-\theta^*\|_2^2~.
\end{align}

We need the above equation to be satisfied $\forall \hat{\theta}_t-\theta^* \in E_{r,t}$ for
Theorem~\ref{theo:ellipsoid_bound:sec:analysis}
to hold. Note, the restricted error set has dependence on $t$
because at each round we compute a new estimate $\hat{\theta}_t$. Refer to
(\ref{eq:E_r:sec:est}) in Section~\ref{sec:est} to review the definition and
see why we need to make this distinction. To that end, we
consider the following problem
\begin{equation} \label{eq:re_cone:sec:analysis}
\inf_{\hat{\theta}_t-\theta^* \in \text{cone}(E_{r,t})} \frac{1}{t} \|X_t (\hat{\theta}_t-\theta^*)\|_2^2 \geq \kappa \|\hat{\theta}_t-\theta^*\|_2^2~.
\end{equation}
Clearly if (\ref{eq:re_cone:sec:analysis}) is true then it is true for all
$\hat{\theta}_t-\theta^* \in E_{r,t}$ since $E_{r,t} \subseteq \text{cone}(E_{r,t})$.
Additionally, since only the direction matters
and not the magnitude we consider just the vectors on the spherical cap
$A_t \defeq \text{cone}(E_{r,t}) \cap S^{p-1}$
\begin{equation}
\inf_{u \in A_t} \frac{1}{t} \|X_t u\|_2^2 \geq \kappa \|u\|_2^2
\end{equation}
where $S^{p-1}$ is the unit sphere in $\mathbb{R}^p$.
Since $\|u\|_2 = 1$ for all $u \in A_t$ we simply focus on
\begin{equation}
\inf_{u \in A_t} \frac{1}{t} \|X_t u\|_2^2 \geq \kappa
\end{equation}
which suffices in proving the RE condition for the restricted error set.

Now, to show a bound,
we perform the following decomposition. Let $X_t = [x_1, \dots, x_t]^\top$, then
\begin{align*}
\frac{1}{t} \|X_t u\|_2^2 = \frac{1}{t} \sum_{i=1}^{t} \langle x_i,u \rangle^2 
 &= \frac{1}{t} \sum_{i=1}^t \langle x_i - \mu_i,u \rangle^2 - \frac{1}{t} \sum_{i=1}^t \langle \mu_i,u \rangle^2 + \frac{2}{t} \sum_{i=1}^t \langle x_i,u \rangle \langle \mu_i,u \rangle \\
 &= \frac{1}{t} \sum_{i=1}^t \langle x_i - \mu_i,u \rangle^2 + \frac{2}{t} \sum_{i=1}^t \langle \mu_i,u \rangle \langle x_i-\mu_i,u \rangle + \frac{1}{t} \sum_{i=1}^t \langle \mu_i,u \rangle^2
\end{align*}
where $\mu_i = \mathbb{E}[x_i|F_{i-1}]$ and we define the filtration to be
$F_{i-1}=\{x_1,\dots,x_{i-1},\eta_1,\dots,\eta_{i-1}\}$.
Taking the infimum
\begin{align}
\label{eq:xt_mut_inf_bound}
\inf_{u \in A_t} \frac{1}{t} \|X_t u\|_2^2 &= \inf_{u \in A_t} \frac{1}{t} \sum_{i=1}^t \langle x_i - \mu_i,u \rangle^2 + \inf_{u \in A_t} \frac{2}{t} \sum_{i=1}^t \langle \mu_i,u \rangle \langle x_i-\mu_i,u \rangle + \frac{1}{t} \sum_{i=1}^t \langle \mu_i,u \rangle^2 \notag \\
 &\geq \inf_{u \in A_t} \frac{1}{t} \sum_{i=1}^t \langle x_i - \mu_i,u \rangle^2 - \sup_{u \in A_t} \frac{2}{t} \sum_{i=1}^t \langle x_i-\mu_i,u \rangle
\end{align}
where the inequality follows from $|\langle \mu_i,u \rangle| \leq 1$ due to
our assumption that $\mathcal{X} \subseteq B_2^p$ to avoid scaling factors.
Suitable scaling modifications can be made to remove the assumption.
To obtain the bounds we have to bound the quantities
$\inf_{u\in A_t} \frac{1}{t} \sum_{i=1}^t \langle x_i - \mu_i,u \rangle^2$ and
$\sup_{u\in A_t} \frac{2}{t} \sum_{i=1}^t \langle x_t-\mu_i,u \rangle$.

\medskip

\noindent \textbf{1. Bound for $\sup_{u \in A_t} \frac{2}{t} \sum_{i=1}^{t} \langle x_i - \mu_i,u \rangle$}

Observe, for all $i$ that $x_i - \mu_i$ is a bounded vector-valued MDS such that
$\subgnorm{x_i-\mu_i} \leq K$ (see Definition~\ref{defn:subg_norm_constant:sec:asmps_defns}).
Therefore, by the Azuma-Hoeffding inequality we obtain
\begin{align}
P \left ( \frac{1}{\sqrt{t}}\left | \sum_{i=1}^t \langle x_i - \mu_i,u \rangle \right | \geq \gamma \right ) \leq 2 \exp \left ( \frac{-\gamma^2}{2 \|u\|_2^2K^2} \right )~.
\end{align}
Therefore, for $u,v \in A_t$
\begin{equation}
P \left ( \frac{1}{\sqrt{t}}\left | \sum_{i=1}^t \langle x_i - \mu_i,u - v \rangle \right | \geq \gamma \right ) \leq 2 \exp \left ( \frac{-\gamma^2}{2\|u - v\|_2^2K^2} \right )~.
\label{eq:xt-mut_subg_conc}
\end{equation}
From $\myref{eq:xt-mut_subg_conc}$ and using the generic chaining
argument~\citep{tala14} similar to our $\lambda_t$ analysis (Section~\ref{sec:appendix_lambda_t})
it follows that for an absolute constant $L>0$,
\begin{equation}
2\mathbb{E} \left [ \sup_{u \in A_t} \frac{1}{\sqrt{t}} \sum_{i=1}^{t} \langle x_i - \mu_i,u \rangle \right ] \leq 2LKw(A_t)~.
\end{equation}
Therefore,
\begin{equation}
P \left ( \sup_{u \in A_t} \left | \frac{1}{\sqrt{t}} \sum_{i=1}^t \langle x_i - \mu_i,u \rangle \right | \geq 2LKw(A_t) + \alpha \right ) \leq L \exp \left ( - \left ( \frac{\alpha}{LK\phi(A_t)}  \right )^2 \right )~.
\end{equation}
Setting $\alpha = LKw(A_t) \zeta$ gives
\begin{equation}
P \left( \sup_{u \in A_t} \left| \frac{1}{\sqrt{t}} \sum_{i=1}^t \langle x_i - \mu_i,u \rangle \right | \geq LKw(A_t)(2+\zeta) \right) \leq L \exp \left( -\left(\frac{\zeta w(A_t)}{\phi(A_t)}\right)^2 \right)~.
\end{equation}

Now, since the set $A_t$ changes each round and the bound must hold across
all rounds, we put the bound in terms of the largest spherical cap $A_{\max}$
which is defined in Definition~\ref{defn:A:sec:asmps_defns}.
Then, setting
$\zeta^2 = \left(\epsilon^2\phi^2(A_{\max})+\log{T}\left(\frac{\phi^2(A_{\max})}{w^2(A_{\max})}\right)\right)$,
and taking a union bound such that across all rounds we have
\begin{align*}
&\bigcup_{t=1}^T P \left( \sup_{u \in A_t} \left| \frac{1}{\sqrt{t}} \sum_{i=1}^t \langle x_i - \mu_i,u \rangle \right | \geq LKw(A_{\max})\left(2+\sqrt{\epsilon^2\phi^2(A_{\max})+\log{T}\left(\frac{\phi^2(A_{\max})}{w^2(A_{\max})}\right)} \; \right) \right) \\
&\leq L \sum_{t=1}^T \exp \left( -\epsilon^2\phi^2(A_{\max})\left(\frac{w^2(A_{\max})}{\phi^2(A_{\max})}\right) - \log{T}\left(\frac{\phi^2(A_{\max})}{w^2(A_{\max})}\right) \left(\frac{w^2(A_{\max})}{\phi^2(A_{\max})}\right) \right) \\
&= L \sum_{t=1}^T \exp \left( -\epsilon^2w^2(A_{\max}) - \log{T} \right) \\
&= L \sum_{t=1}^T \exp \left( -\epsilon^2w^2(A_{\max})\right) \times \frac{1}{T} \\
&= L \exp \left( -\epsilon^2w^2(A_{\max})\right)~.
\end{align*}

Dividing the other $\sqrt{t}$ through and multiplying by $2$ we obtain
\begin{equation}
\label{eq:xt_mut_linear_bound}
\sup_{u \in A_t} \frac{2}{t} \sum_{i=1}^t \langle x_i - \mu_i,u \rangle \leq \frac{2LKw(A_{\max})\left(2+\sqrt{\epsilon^2\phi^2(A_{\max})+\log{T}\left(\frac{\phi^2(A_{\max})}{w^2(A_{\max})}\right)}\right)}{\sqrt{t}}
\end{equation}
which holds uniformly across all rounds with probability at least 
$1-L \exp \left( -\epsilon^2w^2(A_{\max})\right)$.

\medskip

\noindent \textbf{2. Bound for $\inf_{u \in A_t} \frac{1}{t} \sum_{i=1}^t \langle x_i - \mu_i,u \rangle^2$} \\

To prove a bound on $\inf_{u \in A_t} \frac{1}{t} \sum_{i=1}^t \langle x_i-\mu_i,u \rangle^2$
we use the following result from~\citep{bcfs15} which we extend to martingales
in Section~\ref{ssec:appendix_mend}.
\begin{lemm} \textbf{(\citep{bcfs15}, Theorem 12)}
\label{theo:mds_anisotropic_re}
Let $Z_t \in \mathbb{R}^{t \times p}$ be a design matrix with bounded
martingale difference sequence, anisotropic sub-Gaussian rows, i.e.,
$\mathbb{E}[z_i] < \infty, \mathbb{E}[z_i | F_{i-1}] = 0$ where $F_{i-1}$ is a filtration, 
$\mathbb{E}[z_i z_i^\top] = \Sigma$, and
$\subgnorm{z_i \Sigma^{-1/2}} \leq K \; \forall i$.
Then, for absolute constants $c_0,c > 0$, and any $\gamma > 0$ with probability
at least $1-2\exp(-c_0 w^2(A_t)\gamma^2)$, we have
\begin{equation}
\inf_{u \in A_t} \frac{1}{t} \|Z_t u\|^2_2 \geq \lambda_{\min} (\Sigma|A_t) \left ( 1 - c \frac{w(A_t)\gamma}{\sqrt{t}} \right )
\end{equation}
where $\lambda_{\min}(\Sigma|A_t) = \inf_{u \in A_t} u^\top \Sigma u$
is the restricted minimum eigenvalue of $\Sigma$ restricted to $A_t \subseteq S^{p-1}$.
\end{lemm}
Note, we have made a slight change in the theorem as stated in~\cite{bcfs15} to
put the probability in terms of a parameter $\gamma$. We did this by setting
$\theta = c_1c_4\kappa^2 \frac{w(A_t)\gamma}{\sqrt{t}}$ at the bottom of page 31 and
the rest follows through with some algebra. See~\cite{bcfs15} for more details.

Given Lemma~\ref{theo:mds_anisotropic_re}, let $z_i = x_i - \mu_i$ which is
an MDS for all $i$ and then the design matrix $Z_t = [z_1, \dots, z_t]^\top$.
Each row $z_i \in Z_t$ will have a covariance matrix $\mathbb{E}[z_i z_i^\top] = \Sigma_i$
such that $\subgnorm{z_i \Sigma_i^{-1/2}} \leq \|z_i \Sigma_i^{-1/2}\|_2 \leq K$
for some constant $c_1$.

We can immediately apply Lemma~\ref{theo:mds_anisotropic_re} to get a result in terms of
$\lambda_{\min}(\Sigma_i|A_t)$ however, we must be careful to ensure that
$\lambda_{\min}(\Sigma_i|A_t) \neq 0$. Next, we will argue that for all rows of
$Z_t$ that $\lambda_{\min}(\Sigma_i|A_t) > 0$. The following two paragraphs can
be skipped if it is clear that the minimum eigenvalue of such a centered convex
subset with non-empty interior of an $L_2$ ball is positive.

Under the assumption that the decision set $\mathcal{X}$ is compact convex with
non-empty interior
(Assumption~\ref{asmp:decision_set:sec:asmps_defns}), at each round
$t=1,2,\dots,T$ a single solution $x_t'$ is computed via
(\ref{eq:arm_select_round_t:sec:alg}) in Section~\ref{sec:setting} over which
a closed $L_2$ ball with radius $\xi_x>0$ is centered $\bar{B}_2^p(x_t',\xi_x)$.
A single solution $x_t$ is drawn uniformly at random from the set
$\mathcal{X} \cap \bar{B}_2^p(x_t',\xi_x)$ which is a subset of an $L_2$ ball.
Now, if we define the set $\mathcal{Z}_t = \{x-\mu_t: x \in \mathcal{X} \cap \bar{B}_2^p(x_t', \xi_x)\}$
then $\mathcal{Z}_t$ is a subset of an $L_2$ ball centered at the origin.

Given this, we will use a proof by contradiction. If we do not restrict
ourselves to the set $A_t$, clearly
$\lambda_{\min}(\Sigma_t | A_t) \geq \lambda_{\min}(\Sigma_t)$, then we desire a
bound of the form for some $\nu$
\begin{equation*}
\lambda_{\min}(\Sigma_t) = \inf_{u \in S^{p-1}} u^\top \Sigma_t u \geq \nu > 0~.
\end{equation*}
Assume for a moment that $\lambda_{\min}(\Sigma_t) = 0$. Then, compute the
eigenvalue decomposition of $\Sigma_t$ as $\Sigma_t = V \Lambda V^\top$ where
$V = [v_1, \dots, v_p]$ are the eigenvectors of $\Sigma_t$. If we can believe
our assumption that $\lambda_{\min}(\Sigma_t) = 0$ this implies that
$\mathbb{E}_{z \sim \mathcal{Z}_t}[\langle z, v_p \rangle^2] = 0$. If we define
the set $Z_{v_p} = \{z \in \mathcal{Z}_t: \langle z, v_p \rangle = 0 \}$ then for our assumption
to be true it must be true that $P(z \in Z_{v_p}) = 1$ a.s. However, since there
is zero probability density outside of $\mathcal{Z}_t$ this implies the density
is concentrated on a subspace. Such an implication cannot be true because the
span of $\mathcal{Z}_t$ is $\mathbb{R}^p$, i.e., the set contains all directions.
Therefore, our assumption is false and $\lambda_{\min}(\Sigma_t) \neq 0$ which
implies there exists some constant $\nu$ such that
$\lambda_{\min}(\Sigma_t|A_t) \geq \lambda_{\min}(\Sigma_t) \geq \nu > 0$ for
all $t$.

Now, given the argument that $\lambda_{\min}(\Sigma_i)>0 \; \forall i$, we
define $\lambda_{\min}(\Sigma_{1:t}) = \min\{\lambda_{\min}(\Sigma_1), \dots, \lambda_{\min}(\Sigma_t)\}$.
Then, we can use Lemma~\ref{theo:mds_anisotropic_re} with the largest spherical
cap $A_{\max}$ to obtain the following bound which holds for any of the $t$ rounds and
any $\zeta > 0$.
\begin{equation}
P\left( \inf_{u \in A_t} \frac{1}{t} \sum_{i=1}^t \langle x_i - \mu_t, u \rangle^2 \leq \lambda_{\min}(\Sigma_{1:t}) \left(1-c \frac{w(A_{\max})\zeta}{\sqrt{t}} \right) \right) \leq 2 \exp(-c_0 w^2(A_{\max})\zeta^2)~.
\end{equation}

Setting $\zeta^2 = \frac{\epsilon^2}{c_0} + \frac{\log{T}}{c_0w^2(A_{\max})}$ and applying a union bound we
get the following bound which holds simultaneously
for all rounds $t=1,\dots,T$ and any $\epsilon > 0$.
\begin{align}
\label{eq:xt_mut_square_bound}
&\bigcup_{t=1}^T P \left( \inf_{u \in A_t} \frac{1}{t} \sum_{i=1}^t \langle x_i - \mu_t, u \rangle^2 \leq \lambda_{\min}(\Sigma_{1:t}) \left(1-c \frac{w(A_{\max})\sqrt{\frac{\epsilon^2}{c_0} + \frac{\log{T}}{c_0w^2(A_{\max})}}}{\sqrt{t}} \right) \right) \\
&\leq \sum_{t=1}^T 2 \exp\left(-c_0 w^2(A_{\max})\left(\frac{\epsilon^2}{c_0} + \frac{\log{T}}{c_0w^2(A_{\max})}\right)\right) \notag \\
&= \sum_{t=1}^T 2 \exp\left(-w^2(A_{\max})\epsilon^2 - c_0 w^2(A_{\max})\left(\frac{\log{T}}{c_0w^2(A_{\max})}\right)\right) \notag \\
&= \sum_{t=1}^T 2 \exp\left(-w^2(A_{\max})\epsilon^2\right) \times \frac{1}{T} \notag \\
&= 2 \exp\left(-w^2(A_{\max})\epsilon^2\right)~. \notag
\end{align}

Combining (\ref{eq:xt_mut_inf_bound}), (\ref{eq:xt_mut_linear_bound}), and
(\ref{eq:xt_mut_square_bound}) with probability at least
$1- 2L\exp(-w^2(A_{\max})\epsilon^2)$ we obtain
\begin{align}
&\inf_{u \in A_t} \frac{1}{t} \sum_{i=1}^t \|X_t u\|_2^2 \\
&\geq \lambda_{\min}(\Sigma_{1:t}) \left(1-c\frac{w(A_{\max})\sqrt{\frac{\epsilon^2}{c_0} + \frac{\log{T}}{c_0w^2(A_{\max})}}}{\sqrt{t}} \right)
- 2LK\frac{w(A_{\max})\left(2+\sqrt{\epsilon^2\phi^2(A_{\max})+\log{T}\left(\frac{\phi^2(A_{\max})}{w^2(A_{\max})}\right)}\right)}{\sqrt{t}}~. \notag
\end{align}

For some constant $C>0$ it is true that $C-\frac{C}{2}-\frac{C}{4} > 0$
(where we have chosen to divide $C$ by $2$ and $4$ somewhat arbitrarily) therefore,
setting $C = \lambda_{\min}(\Sigma_{1:t})$, if we can show when
\begin{equation} \label{eq:c_2:sec:re}
c\lambda_{\min}(\Sigma_{1:t})\frac{w(A_{\max})\sqrt{\frac{\epsilon^2}{c_0} + \frac{\log{T}}{c_0w^2(A_{\max})}}}{\sqrt{t}} \leq \frac{C}{2}
\end{equation}
and when
\begin{equation} \label{eq:c_4:sec:re}
2LK\frac{w(A_{\max})\left(2+\sqrt{\epsilon^2\phi^2(A_{\max})+\log{T}\left(\frac{\phi^2(A_{\max})}{w^2(A_{\max})}\right)}\right)}{\sqrt{t}} \leq \frac{C}{4}
\end{equation}
then
$\inf_{u \in A_t} \frac{1}{t} \sum_{i=1}^t \|X_t u\|_2^2 > 0$ will be satisfied.

With some algebraic manipulations, we can see that (\ref{eq:c_2:sec:re}) is satisfied
when
\begin{equation*}
t \geq 4c^2\lambda_{\min}^2(\Sigma_{1:t})w^2(A_{\max})\left(\frac{\epsilon^2}{c_0}+\frac{\log{T}}{c_0w^2(A_{\max})}\right)/C^2
\end{equation*}
and (\ref{eq:c_4:sec:re}) is satisfied when
\begin{equation*}
t \geq 64L^2K^2w^2(A_{\max})\left(2+\sqrt{\epsilon^2\phi^2(A_{\max})+\log{T}\left(\frac{\phi^2(A_{\max})}{w^2(A_{\max})}\right)}\right)^2 / C^2~.
\end{equation*}

Therefore, the RE condition will be satisfied when $t \geq c'w^2(A_{\max})(\epsilon^2+\log{T})$
for some constant $c'>0$ which completes the proof.
\end{proof}
\qed

%% file: appendix_mend.tex
We extend~\cite[Theorem 12]{bcfs15} to martingale difference samples
which is an application of Theorem D from~\cite{mept07}. Theorem D relies on
Lemma 1.2~\cite{mept07} which
shows concentrations for i.i.d random samples and is the only part of the
proof which needs to be modified for martingales. As such, we present an extension
of Lemma 1.2 to martingales which generalizes Theorem D to martingales which
can be applied to prove Theorem~\ref{theo:mds_anisotropic_re}. Note, the following
result can be considered independent from the rest of the paper and, as such,
the notation is not inherited but will be re-defined here.
Theorem D is as follows.
\begin{lemm} \textbf{(Mendelson et al. Theorem D)} \label{theo:mend_thrmD}
There exists absolute constants $c_1,c_2$ for which the following holds.
Let $(\Omega,\mu)$ be a probability space, set $F$ be a subset of the unit sphere
of $L_2(\mu)$,i.e., $F \subseteq S_{L_2} = \{f: |\!\|f\|\!|_{L_2} = 1\}$, and
assume that $\emph{diam}(F, \|\cdot\|_{\psi_2}) = \alpha$.
Then, for any $\theta > 0$ and $n \geq 1$ satisfying
\begin{equation}
c_1 \alpha \gamma_2(F,\|\cdot\|_{\psi_2}) \leq \theta \sqrt{n}
\end{equation}
with probability at least $1-\exp\left(-c_2 \frac{\theta^2 n}{\alpha^4}\right)$,
\begin{equation}
\sup_{f \in F} \left|\frac{1}{n} \sum_{i=1}^n f^2(x_i) - \mathbb{E}[f^2] \right| \leq \theta~.
\end{equation}
\end{lemm}

In~\cite{mept07}, the proof assumes the samples $x_1, \dots, x_n \in \mathbb{R}^p$
are i.i.d. isotropic sub-Gaussian random vectors. For the problem we are
considering, the samples $x_1, \dots, x_n \in \mathbb{R}^p$ are a bounded
martingale difference sequence (MDS), i.e., $\|x_i\|_2 \leq A$,
$\mathbb{E}[x_i] \leq \infty, \mathbb{E}[x_i | \mathcal{F}_{i-1}] = 0$ where
$\mathcal{F}_i = \{x_1, \dots, x_i\}$ is a filtration. Also we consider the
following class of functions:
\begin{equation}
F \defeq \left \{ f_u: f_u(\cdot) = \frac{1}{\sqrt{\mathbb{E}[\langle x_i,u \rangle^2|\mathcal{F}_{i-1}]}} \langle \cdot,u \rangle = \frac{1}{\sqrt{u^T \Sigma_i u}} \langle \cdot,u \rangle, u \in A \subseteq S^{p-1} \right\}~.
\end{equation}

We allow the covariance matrix
$\Sigma_i = \mathbb{E}[x_i x_i^\top|\mathcal{F}_{i-1}]$ to be different for each
sample $x_i$. The main lemma we will prove is as follows.
\begin{lemm} \label{theo:sup_mds}
For a bounded martingale difference sequence $x_1,\dots,x_n \in \mathbb{R}^p$ where
each $x_i$ is bounded as $\|x_i\|_2 \leq A$, there exists absolute constants
$c_1,c_2 > 0$ for which the following holds. Let
$F$ be the set of linear functionals over the unit sphere
$F \defeq \left\{\frac{1}{\sqrt{u^T \Sigma_i u}}\langle \cdot, u \rangle: u \in A \subseteq S^{p-1} \right\}$
and assume that
$\emph{diam}(F, \|\cdot\|_{\psi_2}) = \alpha$. Then, for any $\theta > 0$ and
$n \geq 1$ satisfying
\begin{equation}
c_1 \alpha \gamma_2(F,\|\cdot\|_{\psi_2}) \leq \theta \sqrt{n}
\end{equation}
with probability at least 
$1-\exp\left(-c_2 \frac{\theta^2 n}{\alpha^4}\right)$,
\begin{equation}
\sup_{u \in S^{p-1}} \left|\frac{1}{n} \sum_{i=1}^n \frac{\langle x_i, u \rangle^2}{u^T \Sigma_i u} - 1 \right| \leq \theta
\end{equation}
where $X = [x_1, \dots x_n]^\top \in \mathbb{R}^{n \times p}$ is the design matrix
and $\Sigma_i = \mathbb{E}[x_i x_i^\top] \in \mathbb{R}^{p \times p}$ is the
population covariance matrix for sample $x_i$.
\end{lemm}
The proof follows using analogous arguments from~\cite{mept07}.
For $f_u \in F$, we define the random variables $Z_{f_u}$ and $W_{f_u}$ as
\begin{equation}
Z_{f_u} = \frac{1}{n} \|u\|_2^2 \sum_{i=1}^n f^2(x_i) - \mathbb{E}[f^2] = \frac{1}{n} \|u\|_2^2 \sum_{i=1}^n \left( \frac{\langle x_i,u \rangle^2}{u^T \Sigma_i u} -1 \right )~,
\end{equation}
\begin{equation}
W_{f_u} = \left( \frac{1}{n} \|u\|_2^2 \sum_{i=1}^N f^2(x_i) \right)^{1/2} = \left ( \frac{1}{n} \|u\|_2^2 \sum_{i=1}^n \frac{\langle x_i,u \rangle^2}{u^T \Sigma_i u} \right )^{1/2}~.
\end{equation}

We prove the following lemma, which is an analogous result to Lemma 1.2
in~\cite{mept07}.

\begin{lemm}\label{lemm:main}
There exists an absolute constant $c_1 > 0$ for which the following holds. For every $f_u,f_v \in F$ and every $\epsilon > 2$ we have,
\begin{equation} \label{eq:wu_wv}
P(W_{f_u - f_v} \geq \epsilon \|f_u - f_v\|_{\psi_2} ) \leq 2 \exp(-c_1 n \epsilon^4)~.
\end{equation}
Also, for every $u >0 $,
\begin{equation}
\label{eq:sq_eq}
P \left ( |Z_{f_u}| \geq \epsilon \alpha^2  \right ) \leq 2 \exp (-c_1 n \epsilon^2)~,
\end{equation}
\begin{equation}
\label{eq:sq_eq_diff}
P(|Z_{f_u} - Z_{f_v}| \geq \epsilon \alpha \|f_u - f_v\|_{\psi_2}) \leq 2 \exp(-c_1 n \epsilon^2)~.
\end{equation}
\end{lemm}

\begin{proof}
First, we prove (\ref{eq:wu_wv}). The process $W_{f_u - f_v}$ is defined as follows:
\begin{equation}
W_{f_u - f_v} = \left ( \frac{1}{n} \|u-v\|_2^2  \sum_{i=1}^n \frac{\langle x_i,u-v \rangle^2}{(u-v)^T \Sigma_i (u-v)} \right )^{1/2}   ~.
\end{equation}
We define the random variable
$z_i = \frac{\langle x_i,u-v \rangle^2}{(u-v)^T \Sigma_i (u-v)} - \mathbb{E} \left [ \frac{\langle x_i,u-v \rangle^2}{(u-v)^T \Sigma_i (u-v)} |  \mathcal{F}_{i-1} \right ]$
which implies $z_i$ is a bounded MDS as $\mathbb{E}[z_i|\mathcal{F}_{i-1}] = 0$ and
$|z_i| \leq c A^2 = B$ for some constant $c$. Therefore, by the Azuma-Hoeffding
inequality,
\begin{align*}
&P \left ( \frac{1}{n} \left | \sum_{i=1}^n z_i \right | \geq t \right ) \leq 2 \exp \left ( \frac{nt^2}{2B^2} \right ) \\
= &P \left ( \frac{1}{n} \left | \sum_{i=1}^n \frac{\langle x_i,u-v \rangle^2}{(u-v)^T \Sigma_i (u-v)} - \mathbb{E} \left [ \frac{\langle x_i,u-v \rangle^2}{(u-v)^T \Sigma_i (u-v)} \Big|  \mathcal{F}_{i-1} \right ] \right | \geq t \right ) \\
= &P \left ( \frac{1}{n} \left | \sum_{i=1}^n \frac{\langle x_i,u-v \rangle^2}{(u-v)^T \Sigma_i (u-v)} - 1 \right | \geq t \right ) \\
= &P \left ( \frac{1}{n} \|u-v\|_2^2 \left | \sum_{i=1}^n \frac{\langle x_i,u-v \rangle^2}{(u-v)^T \Sigma_i (u-v)} - 1 \right | \geq t\|u-v\|_2^2 \right ) \\
= &P \left ( W_{f_u-f_v}^2 - \|u-v\|_2^2 \geq t\|u-v\|_2^2 \right ) \\
= &P \left ( W_{f_u-f_v}^2 \geq c_2B^2\epsilon^2\|u-v\|_2^2 \right ) \quad \text{(Setting $t=c_2B^2\epsilon^2-1$)}\\
= &P \left ( W_{f_u-f_v}^2 \geq \epsilon^2\|f_u-f_v\|_{\psi_2}^2 \right ) \\
= &P \left ( W_{f_u-f_v} \geq \epsilon\|f_u-f_v\|_{\psi_2} \right ) \leq 2\exp \left( -c_1 n \epsilon^4 \right) ~.
\end{align*}
The first equality follows by replacing the value of $z_i$,
the second by noting that
$\mathbb{E} \left [ \frac{\langle x_i,u-v \rangle^2}{(u-v)^T \Sigma_i (u-v)} |  \mathcal{F}_{i-1} \right ] = 1$,
the third by multiplying both sides by $\|u-v\|_2^2$,
the fifth by setting $t = c_2 \epsilon^2 B^2 - 1$ for some constant $c_2 > 0$
and noting that
$\|f_u - f_v \|_{\psi_2}^2 = \|u-v\|_2^2 \sup_{\frac{u-v}{\|u-v\|_2} \in S^{p-1}} \left \langle x_i,\frac{u-v}{\|u-v\|_2} \right \rangle^2 = c_2 B^2\|u-v\|_2^2$,
and then taking the square root.

We use similar arguments to prove (\ref{eq:sq_eq}).
Let $z_i = \frac{\langle x_i,u \rangle^2}{u^T \Sigma_i u} - \mathbb{E}\left[\frac{\langle x_i,u \rangle^2}{u^T \Sigma_i u} | F_{i-1} \right]$.
By the argument given earlier, $z_i$ is a bounded MDS and $|z_i| \leq cB$ for
some constant $c$. Using the Azuma-Hoeffding inequality we obtain
\begin{align*}
&P \left ( \frac{1}{n} \|u\|_2^2 \left | \sum_{i=1}^n z_i \right | \geq t \right ) \\
= &P \left ( \frac{1}{n} \|u\|_2^2 \left | \sum_{i=1}^n \frac{\langle x_i,u \rangle^2}{u^T \Sigma_i u} - \mathbb{E}\left[\frac{\langle x_i,u \rangle^2}{u^T \Sigma_i u} | F_{i-1} \right] \right| \geq t \right ) \\
= &P \left ( \frac{1}{n} \|u\|_2^2 \left | \sum_{i=1}^n \frac{\langle x_i,u \rangle^2}{u^T \Sigma_i u} - 1 \right| \geq t \right ) \\
= &P \left (| Z_{f_u} | \geq t \right ) \leq 2 \exp \left ( - \frac{nt^2}{2c^2B^2} \right )~.
 \end{align*}
Let $t = \epsilon cB$ and noting that $\alpha^2 = c B$ we get the following bound
for some constant $c_1 > 0$
\begin{equation}
P \left ( |Z_{f_u}| \geq \epsilon \alpha^2  \right ) \leq 2 \exp (-c_1 n \epsilon^2)~.
\end{equation}
For the proof of~\myref{eq:sq_eq_diff} note that
\begin{equation}
Z_{f_u} - Z_{f_v} = \frac{1}{n} \sum_{i=1}^n \frac{\langle x_i,u \rangle^2}{u^T \Sigma_i u} - \frac{\langle x_i,v \rangle^2}{v^T \Sigma_i v}~.
\end{equation}
Let
$z_i = \frac{\langle x_i,u \rangle^2}{u^T \Sigma_i u} - \frac{\langle x_i,v \rangle^2}{v^T \Sigma_i v} - \mathbb{E}\left[\frac{\langle x_i,u \rangle^2}{u^T \Sigma_i u} - \frac{\langle x_i,v \rangle^2}{v^T \Sigma_i v}| F_{i-1}\right]$.
Again $z_i$ is a bounded MDS with $|z_i| \leq \alpha \|f_u - f_v\|_{\psi_2}$.
\begin{align*}
&P\left(\frac{1}{n} \left| \sum_{i=1}^n z_i \right| \geq t \right) \\
= &P\left(\frac{1}{n} \left| \sum_{i=1}^n \frac{\langle x_i,u \rangle^2}{u^T \Sigma_i u} - \frac{\langle x_i,v \rangle^2}{v^T \Sigma_i v} - \mathbb{E}\left[\frac{\langle x_i,u \rangle^2}{u^T \Sigma_i u} - \frac{\langle x_i,v \rangle^2}{v^T \Sigma_i v}| F_{i-1}\right] \right| \geq t \right) \\
= &P\left(\frac{1}{n} \left| \sum_{i=1}^n \frac{\langle x_i,u \rangle^2}{u^T \Sigma_i u} - \frac{\langle x_i,v \rangle^2}{v^T \Sigma_i v} \right| \geq t \right) \\
= &P(|Z_{f_u} - Z_{f_v}| \geq \epsilon \alpha \|f_u - f_v\|_{\psi_2}) \leq 2 \exp(-c_1 n \epsilon^2)~.
\end{align*}

The second inequality follows since 
$\mathbb{E}\left[\frac{\langle x_i,u \rangle^2}{u^T \Sigma_i u} - \frac{\langle x_i,v \rangle^2}{v^T \Sigma_i v}| F_{i-1}\right] = 0$
and the last equality follows by setting $t = c_2 B \alpha \|f_u-f_v\|_{\psi_2}\epsilon$.
This concludes the proof of Lemma~\ref{lemm:main} and following the same proof
of Theorem D using Lemma~\ref{lemm:main} instead of Lemma 1.2~\cite{mept07} we
prove Lemma~\ref{theo:sup_mds}.
\end{proof}
\qed

Finally, for the proof of Lemma~\ref{theo:mds_anisotropic_re}, we follow the
same application as in~\cite[Theorem 12]{bcfs15}. This concludes the proof of
Lemma~\ref{theo:mds_anisotropic_re}.